\documentclass[letterpaper]{article} % DO NOT CHANGE THIS
\usepackage{aaai23-r2hcai}  % DO NOT CHANGE THIS
\usepackage{amsmath,amsthm,amssymb,amsfonts}
\usepackage{times}  % DO NOT CHANGE THIS
\usepackage{helvet}  % DO NOT CHANGE THIS
\usepackage{courier}  % DO NOT CHANGE THIS
\usepackage[hyphens]{url}  % DO NOT CHANGE THIS
\usepackage{graphicx} % DO NOT CHANGE THIS
\usepackage{natbib}  % DO NOT CHANGE THIS AND DO NOT ADD ANY OPTIONS TO IT
\usepackage{caption} % DO NOT CHANGE THIS AND DO NOT ADD ANY OPTIONS TO IT
\usepackage{algorithm,algpseudocode}
\algblock{ParFor}{EndParFor}
\algnewcommand\algorithmicparfor{\textbf{Parallel for}}
\algnewcommand\algorithmicpardo{\textbf{do}}
\algnewcommand\algorithmicendparfor{\textbf{end\ for}}
 \algrenewtext{ParFor}[1]{\algorithmicparfor\ #1\ \algorithmicpardo}
\algrenewtext{EndParFor}{\algorithmicendparfor}

\newtheorem{thm}{Theorem}
\newtheorem{asp}[thm]{Assumption}
\newtheorem{prop}[thm]{Proposition}
\newtheorem{lem}[thm]{Lemma}
\newtheorem{cor}[thm]{Corollary}

\newtheorem{defi}[thm]{Definition}

\theoremstyle{remark}

\allowdisplaybreaks

\usepackage{subfig}
\usepackage{newfloat}
\usepackage{listings}
\DeclareCaptionStyle{ruled}{labelfont=normalfont,labelsep=colon,strut=off} % DO NOT CHANGE THIS
\floatstyle{ruled}
\newfloat{listing}{tb}{lst}{}
\floatname{listing}{Listing}
\usepackage{fancyhdr}
\fancypagestyle{firstpagehf}
{
   \fancyhf{}
   \fancyhead[HC]{The AAAI 2023 Workshop on Representation Learning for Responsible Human-Centric AI (R$^2$HCAI)}
   \fancyfoot[C]{\thepage}
}

\title{Quantifying the Impact of Label Noise on Federated Learning}
\author{
    Shuqi Ke$^1$\quad
    Chao Huang$^2$\quad
    Xin Liu$^2$
}
\affiliations {
    \textsuperscript{\rm 1} The Chinese University of Hong Kong, Shenzhen \quad
    \textsuperscript{\rm 2} University of California, Davis\\
    shuqike@link.cuhk.edu.cn, fchhuang@ucdavis.edu, xinliu@ucdavis.edu
}

\usepackage{bibentry}
\begin{document}
\thispagestyle{firstpagehf}
\maketitle

\begin{abstract}
Federated Learning (FL) is a distributed machine learning paradigm where clients collaboratively train a model using their local (human-generated) datasets. While existing studies focus on FL algorithm development to tackle data heterogeneity across clients, the important issue of data quality (e.g., label noise) in FL is overlooked. This paper aims to fill this gap by providing a quantitative study on the impact of label noise on FL. We derive an upper bound for the generalization error that is \textit{linear} in the clients' label noise level. Then we conduct experiments on MNIST and CIFAR-10 datasets using various FL algorithms. Our empirical results show that the global model accuracy linearly decreases as the noise level increases, which is consistent with our theoretical analysis. We further find that label noise slows down the convergence of FL training, and the global model tends to overfit when the noise level is high.
\end{abstract}

\section{Introduction}

Federated Learning (FL) is a distributed machine learning paradigm where clients (e.g., distributed devices or organizations) collaboratively train a global model \cite{b19}. The local data of the clients
are often human-generated and have critical privacy concerns. An FL process consists of some communication rounds. In each round, each client trains its local model with its local data and then uploads the model updates to a central server \cite{b15}. The central server aggregates the local updates from clients and sends back an aggregated global model to all clients. After that, clients update their local models according to the information from the central server \cite{b20}. The client-server interaction stops when the global model converges. 

There has been an increasing volume of research studies on FL over the last few years \cite{b19,aFieldGuide,NUSFLsurvey,hkustFLsurvey}. Among these studies, a critical bottleneck, which without appropriate algorithmic treatment usually fails FL, is data heterogeneity (non-IID). For example, in a classification task, some clients may collect more data for class $A$ while others may collect more data for class $B$. Previous studies among this line focused on two categories of non-IID: attribute skew and label skew \cite{b21}. Attribute skew refers to the case where the feature distribution of each client is different from one another. For example, attribute skew could occur in a handwritten digit classification task as users may write the same digit with different font styles, sizes, and stroke widths \cite{b19}. Label skew refers to the case where the label distribution of each client is different from one another. Label skew, for example, could occur in an animal recognition task. Label distributions are different because clients are in different geo-regions and different animal habitats --- dolphins only live near coastal regions, or aquariums \cite{b19}.

While existing studies focus on tackling the non-IIDness, some implicitly assume that the data are clean, i.e., the data are correctly labeled. In practical applications, however, clients' datasets usually contain noisy labels \cite{b4}.  Label noise has been identified in many widely used FL datasets, including MNIST \cite{b3,b8}, EMNIST \cite{b6,b9}, CIFAR-10 \cite{b5,b6}, ImageNet \cite{b3,b7}, and Clothing1M \cite{b38}. The causes of label noise can be human error, subjective labeling tasks, non-exact data labeling processes, and malfunctioning data collection infrastructure \cite{b22,b30}. Moreover, in an FL setting, as clients collect and label local data in a distributed and private fashion, their labels are likely to be noisy and have different noise patterns \cite{b31}. For example, wearable devices can access various human-generated data, such as heart rate, sleep patterns, medication records, and mental health logs. Such data could contain different levels of label noise due to various sensor precision issues and human bias \cite{human_affective_wearable}.

Label noise is known to lessen model performance \cite{b22}. This paper focuses on the issue of label noise in FL, and we are particularly interested in answering the following two key questions:
\begin{itemize}
\item \textbf{Question 1}: \textit{How does label noise affect FL convergence?}
\item \textbf{Question 2}: \textit{How does label noise affect FL generalization?}
\end{itemize}
To answer Question 1, we conduct numerical experiments and show that the training loss converges slower with a higher noise level. To answer Question 2, we proceed from both theoretical and empirical perspectives. First, under minor assumptions, we prove that, for any distributed learning algorithm, the generalization error of the global model is linearly bounded above by a multiple of the system noise level. Then we conduct experiments using MNIST and CIFAR-10, showing that the results are consistent with the assumptions and theoretical results. We further show that the global model's accuracy decreases linearly in the clients' label noise level.

The key contributions of this paper are summarized below.
\begin{itemize}
\item To the best of our knowledge, this is the first quantitative study that analyzes the impact of label noise on FL. Our study bears practical significance for its use in different applications, e.g., incentive design \cite{b34}.
\item We provide a generic upper bound on the FL generalization error that applies to any FL algorithms. We further obtain a tighter upper bound considering the widely adopted ReLU networks in clients' local models.
\item We run experiments under various algorithms and different settings in FL. Our numerical results justify our theoretic assumption. We also observe that label noise linearly degrades FL performance by reducing the test accuracy of the global model.
\item Our study reveals several important and interesting insights. (1) Label noise slows down FL convergence; (2) label noise induces overfitting to the global model; (3) among three benchmark FL algorithms, SCAFFOLD \cite{b14} achieves the best test accuracy than other algorithms with minor label noise, while FedNova \cite{b37} achieves the best test accuracy with more extensive label noise.
\end{itemize}

\section{Related Work}

\subsection{Label noise}

Label noise has been an active topic in FL over the last few years. We classify the existing methods into three categories: (1) Some methods apply \textit{noise-tolerant loss functions} to achieve robust performance \cite{b51}.

\noindent
(2) Some methods \textit{distill confident training sample by selection or a weighting scheme} \cite{b30,b36,b39,b40,b41,b42,b43,b45,b46,b48,b49,b50}. Li et al. discovered that label noise might cause overfitting for FedAvg algorithm. However, they did not analytically characterize the hidden linear relation between noise level and the global model's performance.

\noindent
(3) Based on (2), some methods further \textit{correct noisy samples} \cite{b31,b44,b47,b52}. Tsouvalas et al. proposed FedLN that estimates per-client noise level and corrects noisy labels. However, their definition of label noise is limited because they only considered the engineering method to generate label noise as the definition of label noise. They considered a case where conditional distributions\footnote{In some work, the conditional distribution is also referred to as ``feature-to-label mapping''.} $\mathrm{Pr}(\mathrm{label}|\mathrm{feature})$ are the same across clients \cite{b19}. But in practice, the conditional distributions could be different for different clients. We provide a more general definition in this work and fill this gap. Xu et al. studied an FL scenario where different clients have different levels of label noise \cite{b31}. They introduced local intrinsic dimension (LID), a measure of the dimension of the data manifold. They discovered a strong linear relation between cumulative LID score and local noise level. However, their work did not provide either empirical observation or theoretical results on the relation between the global model's performance and local noise level. Moreover, there is no systematic study on how label noise affects FL in terms of convergence and generalization. We bridge this research gap in this work.

\subsection{Path-norm}

This work uses path-norm to measure the global model's generalization ability under label noise. People introduced different measures to explain the generalization ability of neural networks \cite{b27,b56}. Behnam Neyshabur et al. proposed path-norm as a capacity measure for ReLU networks \cite{b53,b54}. Empirical studies showed that path-norm positively correlates with generalization in all categories of hyper-parameter \cite{b56}.

The value of path-norm increases throughout the learning process. E et al. showed that the path-norm increases at most polynomially under centralized training \cite{b28}. In this work, we conduct the first formal study on the evolution of path-norm in FL. This is also the first work that analyzes the generalization ability of models in FL with path-norm proxy. We introduce path-norm proxy to the FL context because this proxy does not require unrealistic assumptions and allows us to characterize a large class of FL algorithms. For example, the assumptions on convexity, smoothness, etc., are no longer necessary in our analysis. Moreover, we have empirically verified our theory based on the definition of path-norm proxy.

\section{Preliminaries and Problem Statement}\label{sec:prelim}

\subsection{Federated Learning}

\begin{figure}
    \centering
    \includegraphics[width=6cm]{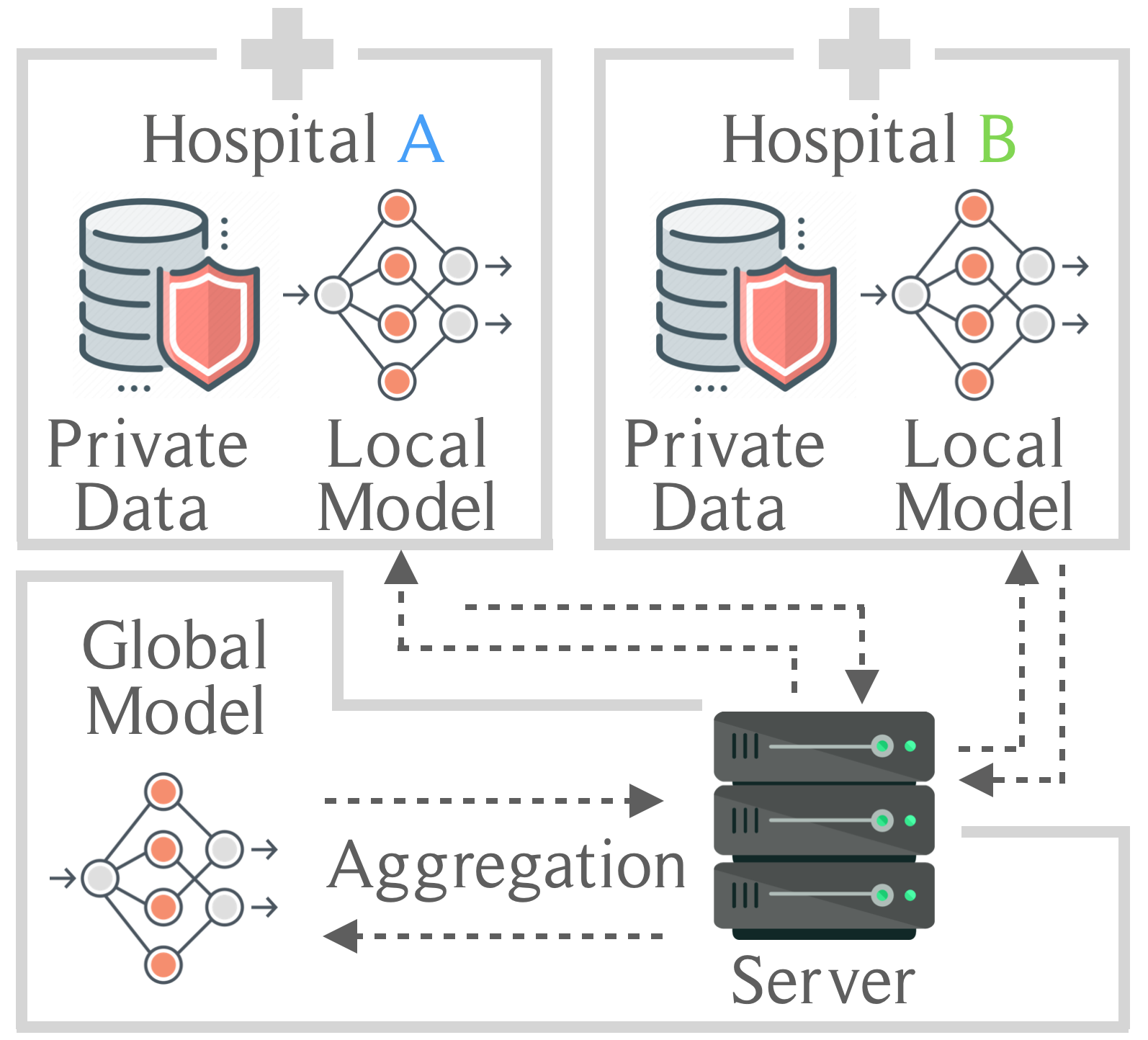}
    \caption{An example of Federated Learning application}
\end{figure}

In this subsection, we briefly introduce the problem formulation and algorithmic framework of FL.

Consider a typical FL task \cite{b19}, where $N$ clients collaboratively train a global model under the coordination of a central server through $R$ communication rounds. FL aims to solve a distributed optimization problem with distributed data. Here we first introduce the objective of the distributed optimization problem and then define the relevant notations. The objective is
\begin{equation}
    \min_{W\in\mathcal{W}}\frac{1}{N}\sum_{k=1}^{N}\left[\frac{1}{n_k}\sum_{i=1}^{n_k}\ell(f(x_{k,i};W),y_{k,i})\right]
\end{equation}
where we define
\begin{itemize}
    \item Hypothesis space: $\mathcal{W}\subset\mathbb{R}^{d_w}$ denotes the hypothesis space of all feasible parameters of learning models, and $d_w\in\mathbb{N}$ is the dimension of the hypothesis space.
    \item Local data: Each client has a local dataset $S_{k}$. We assume that in the $k$-th dataset $S_{k}$, each data point is drawn from a distribution $\pi_{k}$ over $\mathcal{S}\subset\mathbb{R}^{d_x+d_y}$ where $d_x$ denotes the dimension of feature space and $d_y$ denotes the dimension of the label space. A data point $(x,y)\in\mathbb{R}^{d_x+d_y}$ is a real-valued vector where $x\in\mathbb{R}^{d_x}$ denotes its feature and $y\in\mathbb{R}^{d_y}$ denotes its label. There are in total $n_{k}$ data points in client $k$'s local dataset
    \[
        S_{k}=\{(x_{k,1},y_{k,1}),(x_{k,2},y_{k,2}),\dots,(x_{k,n_{k}},y_{k,n_{k}})\}
    \]
    Let $\mu_k$ denote the ground truth distribution (i.e., clean labels) and $\pi_k$ denote client $k$'s  possibly noisy data distribution. There exists label noise in the local dataset of client $k$ if there exists $x\in\mathbb{R}^{d_x},y\in\mathbb{R}^{d_y}$ such that
    \begin{equation}
        \mathrm{Pr}_{\mu_k}(y|x)\not=\mathrm{Pr}_{\pi_k}(y|x)
    \end{equation}
    where $\mathrm{Pr}$ represents a probability mass/density function with a given distribution and an event. One can consider the data points sampled from $\pi_k$ as training data and those sampled from $\mu_k$ as test data.
    \item Global parameter and local parameter: We denote the global model's parameter as a real-valued vector $W\in\mathcal{W}$. Each client has a local model with parameter $w_k\in\mathcal{W}$.
    \item Meta model: We define the meta model $f:\mathbb{R}^{d_x}\times\mathcal{W}\rightarrow\mathbb{R}^{d_y}$ as a function that maps the data feature and model parameter to an estimated label. For example, a meta model could be a neural network with variable parameters. We obtain a model by substituting the variable parameters with real number values.
    \item Loss function: We denote the loss function as
    \[
        \ell:\mathbb{R}^{d_y}\times\mathbb{R}^{d_y}\rightarrow\mathbb{R}_{\ge 0}
    \]
    For example, a squared loss function is defined as $\ell:(y,\hat{y})\mapsto\lVert y-\hat{y}\rVert^2$.
\end{itemize}
 
In each communication round, a client trains its local model for $E$ epochs to minimize the local training loss $\frac{1}{n_k}\sum_{i=1}^{n_k}\ell(f(x_{k,i};W),y_{k,i})$ over its local dataset $S_k$. After local model training, the clients upload their local model parameters $w_k$ to a central server. The central server aggregates the uploaded parameters and updates the global model's parameter $W$. After that, the central server sends the global model's new parameter back to each client. We provide a general FL framework in Algorithm~\ref{alg:fed-frame}.

\begin{algorithm}
    \caption{A General FL Framework}
    \label{alg:fed-frame}
    \begin{algorithmic}[1]
        \renewcommand{\algorithmicrequire}{\textbf{Initialization:}}
        \renewcommand{\algorithmicensure}{\textbf{Output:}}
        \Require Local datasets $\{S_{1},S_{2},\dots,S_{N}\}$, aggregation function $\phi$
        \Ensure  Global model parameter vector $W$ and local model parameter vectors $\{w_{1},w_{2},\dots,w_{N}\}$ after the $R$-th communication round
        \For{t $\leftarrow 1$ to $R$}
            \ParFor{k $\leftarrow 1$ to $N$}
                \For{i $\leftarrow 1$ to $E$}\algorithmiccomment{local training}
                    \State Update local model parameter $w_k$
                \EndFor
                \State Send $w_k$ to the central server
            \EndParFor
            \State $W\leftarrow\phi(w_1,\dots,w_N,W)$\algorithmiccomment{aggregation}
            \For{k $\leftarrow 1$ to $N$}\algorithmiccomment{broadcast}
                \State Send $W$ to client $k$
                \State Update local model parameter $w_k$ according to $W$
            \EndFor
        \EndFor
    \end{algorithmic}
\end{algorithm}

Different FL algorithms use different aggregation mechanisms. We use FedAvg as an example to explain the aggregation step in Algorithm~\ref{alg:fed-frame}. In FedAvg, the aggregation is defined as
\begin{equation}
    \phi:(w_1,\dots,w_N,W)\mapsto(1-\eta_{\mathrm{gl}})W+\eta_{\mathrm{gl}}\frac{\sum_{k=1}^{N}w_k}{N}
\end{equation}
where $\eta_{\mathrm{gl}}$ denotes the global learning rate. Note that in a realistic setting, there could be limitations on computation and communication, including computational efficiency, communication bandwidth, and network robustness \cite{FLframework_nips2020}. For example, some clients may fail to communicate with the central server due to network issues. Therefore, the server only samples a subset of available clients. Since we focus on data noise, we ignore these realistic considerations and assume that all clients participate in all communication rounds.

\subsection{Model performance}

This subsection introduces the theoretical tools to measure a learning algorithm's performance. Here we inherit most notations from the last part with some revisions. We consider fixed data points for a FL process in the previous part. But in this part, we consider each data point and each local dataset $S_k$ as random variables to investigate the generalization performance of an algorithm given an arbitrary training dataset. The pair $(x,y)$ in lowercase represents a deterministic data point, and the pair $(X,Y)$ in uppercase represents pair of random variables. We re-write a local dataset $S_k$ as
\[
    S_{k}=\{(X_{k,1},Y_{k,1}),(X_{k,2},Y_{k,2}),\dots,(X_{k,n_{k}},Y_{k,n_{k}})\}
\]
where $(X_{k,i},Y_{k,i})\sim\pi_k$. We define the empirical risk $L:\mathcal{W}\rightarrow\mathbb{R}_{\ge 0}$ of the global model as
\begin{align}
    L(W)=&\sum_{k=1}^{N}\frac{n_k}{n}\mathbb{E}_{\pi_k}\left[\ell(f(X;W),Y)\right]
\end{align}
where $n:=\sum_{k=1}^{N}n_k$ and $W$ denotes the parameter of the global model. Given the ground truth distribution $\mu_k$ of each client, we further define the ground-truth risk $L^{\dag}:\mathcal{W}\rightarrow\mathbb{R}_{\ge 0}$ of the global model as
\begin{align}
    L^{\dag}(W)=&\sum_{k=1}^{N}\frac{n_k}{n}\mathbb{E}_{\mu_k}\left[\ell(f(X;W),Y)\right]
\end{align}
Then we define the generalization error of the global model as \cite{b23}
\begin{equation}
    G(W):=\left|L^{\dag}(W)-L(W)\right|
\end{equation}

\subsection{Path-norm proxy}\label{subsec:relu-network}

\begin{figure}[h]
    \centerline{\includegraphics*[width=7cm]{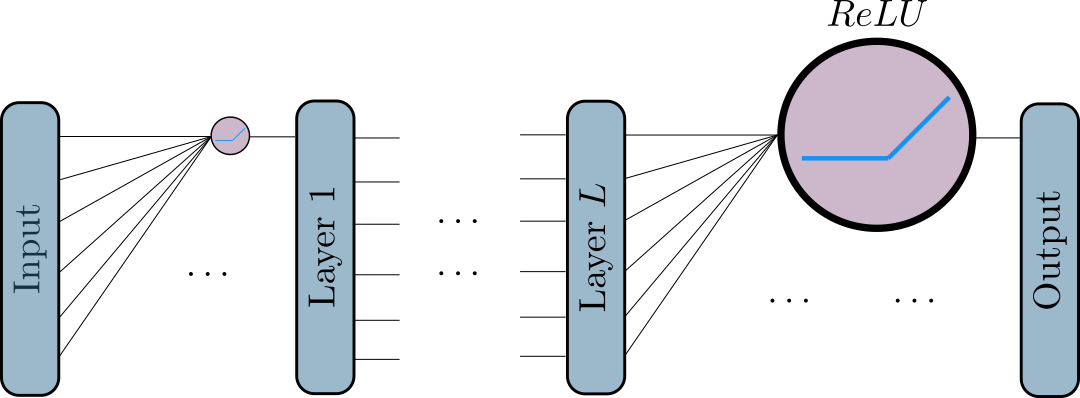}}
    \caption{An $L$-layer ReLU network.}
    \label{fig:relu-vis}
\end{figure}

This paper uses ReLU network (Figure~\ref{fig:relu-vis}) and path-norm proxy (Figure~\ref{fig:path-norm-vis}) for a case study of the generalization error. The authors in \cite{b28} provided a mathematical description of ReLU networks. Based on their definitions, we define the path-norm proxy below.

\begin{defi}[Path-norm proxy {\cite{b28}}]
    The path-norm proxy of an $L$-layer ReLU network is defined as
    \begin{equation}
        \lVert f(\cdot;\theta)\rVert_{\mathrm{pnp}}=\sum_{(i_0,\dots,i_{L+1})}\prod_{l=0}^{L}\left|\theta_l(i_{l},i_{l+1})\right|
    \end{equation}
    where $\theta$ denotes the parameter vector of the ReLU network; $\theta_l(i_{l},i_{l+1})$ refers to the weight of the edge connecting the $i_l$-th node in layer $l$ and the $i_{l+1}$-th node in layer $l+1$.
\end{defi}

\begin{figure}
    \centering
    \includegraphics[width=7cm]{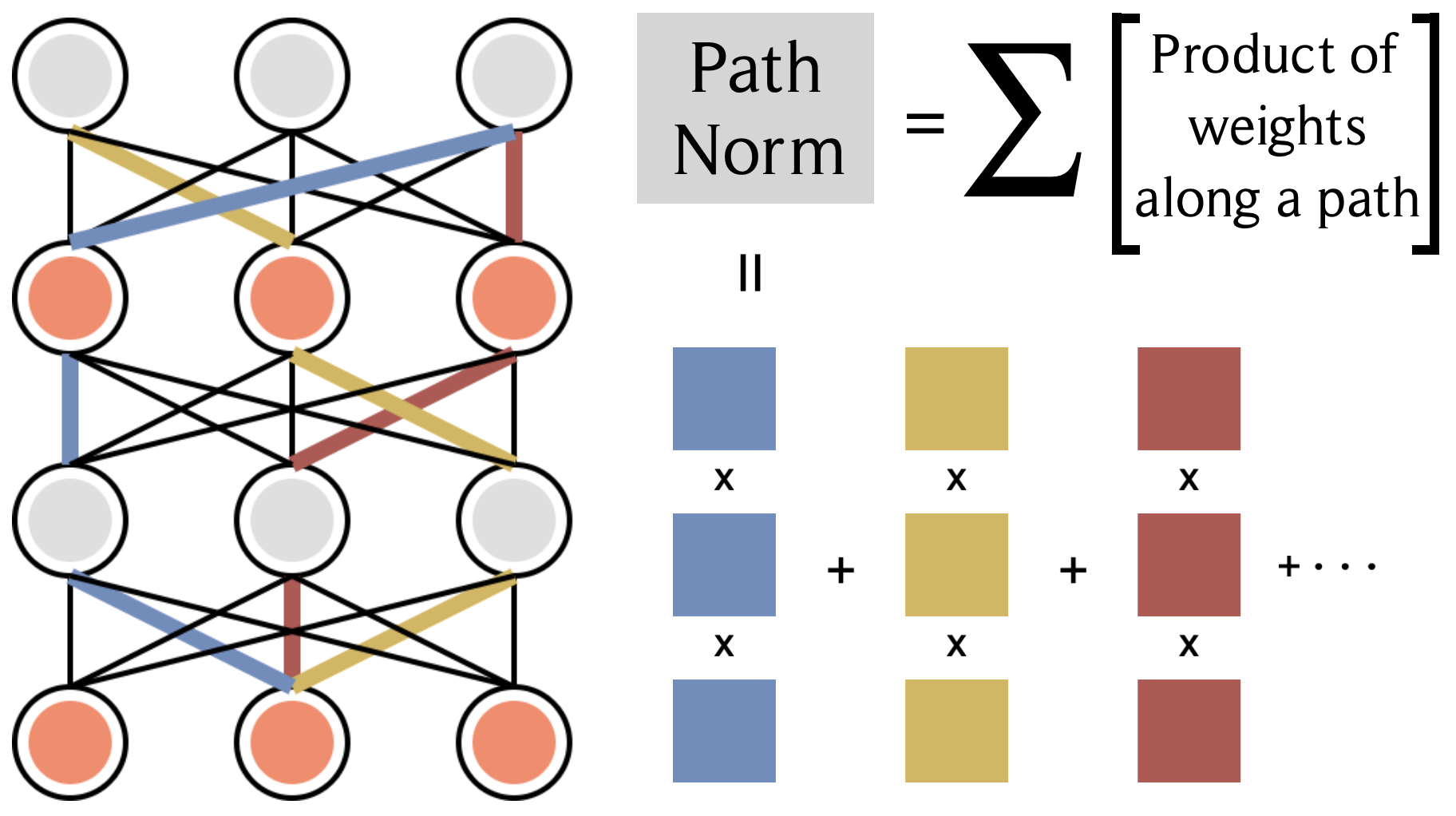}
    \caption{Path-norm proxy}
    \label{fig:path-norm-vis}
\end{figure}

The authors in \cite{b28} also proved that the path norm proxy controls the generalization error in a centralized learning setting. Next, we will show that the path norm proxy controls the generalization error in FL.

\section{Theoretical Results}\label{sec:theory-result}

In this section, we provide a theoretical analysis of the generalization error of the global model in FL. In particular, we give proof of the upper bound of the global model's generalization error.

In practical FL applications, local data distributions are complicated as we cannot explicitly find the distribution functions. To simplify our theoretical analysis, we make the following assumption:

\begin{asp}[Simplified label noise condition]\label{asp:annotation-skew}
    For any client $i$ and client $j$, we assume
    \begin{equation}
        \forall (x,y)\in\mathbb{R}^{d_x+d_y},\mathrm{Pr}(x;\pi_i)=\mathrm{Pr}(x;\pi_j)
    \end{equation}
    % where $\mathrm{Pr}$ represents a probability mass/density function with given a distribution and an event
\end{asp}
\noindent
This assumption assures the feature distributions to be identical for all clients, which is a standard setting in studies about concept drift \cite{2022FL_concept_drift_nips}. Although it is difficult to show that Assumption~\ref{asp:annotation-skew} holds in our experiment settings, the numerical results are still consistent with our theoretical results.

We first provide a general result on the upper bound of generalization error in Theorem~\ref{thm:generalization-error-bound}. Then we extend this general bound by studying some specific cases with more assumptions in Corollary~\ref{cor:explicit-generror-bound}.

\begin{thm}[Bound the evolution of generalization error]\label{thm:generalization-error-bound}
    Consider any Federated Learning algorithm with a neural network with an arbitrary structure for a classification task of $C$ classes under label noise and use the cross-entropy function for loss computation, then under Assumption~\ref{asp:annotation-skew}
    \begin{equation}
        \resizebox{\hsize}{!}{$G(W)\le \Omega \cdot\mathbb{E}_{X}\left[\sum_{i=1}^{C}\sum_{k=1}^{N}\frac{n_k}{n}\left|\mathrm{Pr}_{\mu}(Y=i|X)-\mathrm{Pr}_{\pi_k}(Y=i|X)\right|\right]$}
    \end{equation}
    where $\Omega$ is the upper bound of $f$.
\end{thm}

\noindent
\textbf{Interpretation of Theorem~\ref{thm:generalization-error-bound}:} This theorem implies that the generalization error of global model in FL is linearly bounded by the degree of label noise in the distributed system. The theorem quantitatively characterizes the impact of label noise. This linear bound is also consistent with our empirical findings. When $N=1$, this linear bound applies to centralized learning.

We can interpret the expectation term in the upper bound with an example. In this example, we set $N=2$, i.e., two clients. The input space consists of $25$ discrete grid points and two classes. Client $2$'s local data distribution is identical to the ground truth. Client $1$ has label noise in its local data where three circled data points in class $A$ are mislabelled as class $B$.

\begin{figure}[h]
    \centering
    \subfloat[Data distribution of client 1\label{subfig:labelnoise-ex-client1}]{%
       \includegraphics[width=0.5\linewidth]{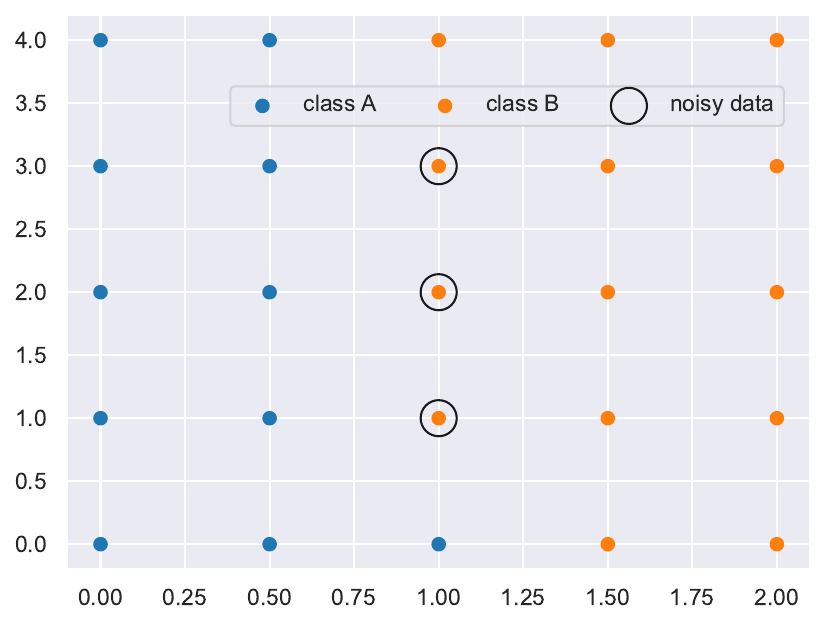}}
  \subfloat[Data distribution of client 2\label{subfig:labelnoise-ex-client2}]{%
        \includegraphics[width=0.5\linewidth]{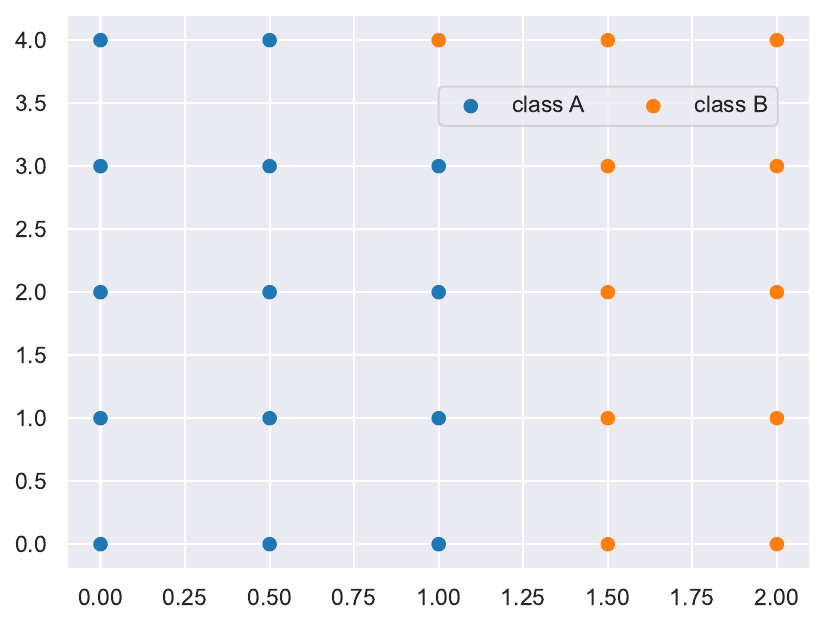}}
  \caption{An example of label noise.}
    \label{fig:label-noise-example}
\end{figure}
\noindent
If the two clients has the same number of data samples, i.e., $n_1=n_2$, then
\begin{equation}
    \begin{aligned}
        &\mathbb{E}_{X}\left[\sum_{i=1}^{C}\sum_{k=1}^{N}\frac{n_k}{n}\left|\mathrm{Pr}_{\mu}(Y=i|X)-\mathrm{Pr}_{\pi_k}(Y=i|X)\right|\right]\\
        &=\mathbb{E}_{X}\left[\sum_{i=1}^{C}\frac{1}{2}\left|\mathrm{Pr}_{\mu}(Y=i|X)-\mathrm{Pr}_{\pi_1}(Y=i|X)\right|\right]\\
        &=\frac{1}{2}\left(\frac{1}{5}\cdot\left|\frac{1}{5}-\frac{4}{5}\right|+\frac{1}{5}\left|\frac{4}{5}-\frac{1}{5}\right|\right)=\frac{3}{25}
    \end{aligned}
\end{equation}
This expectation represents the expected \textbf{percentage of noisy data points} in a dataset, e.g. there are in total $25$ grid points and $3$ noisy data points in Fig~\ref{subfig:labelnoise-ex-client1}.

Now we consider a slightly different example where client $2$ also has label noise as shown in Figure~\ref{fig:label-noise-example2}. Then the expectation is $\frac{4}{25}$.
\begin{figure}[h]
    \centering
    \subfloat[Data distribution of client 1\label{subfig:labelnoise-ex2-client1}]{%
       \includegraphics[width=0.5\linewidth]{img/labelnoiseClient1.pdf}}
  \subfloat[Data distribution of client 2\label{subfig:labelnoise-ex2-client2}]{%
        \includegraphics[width=0.5\linewidth]{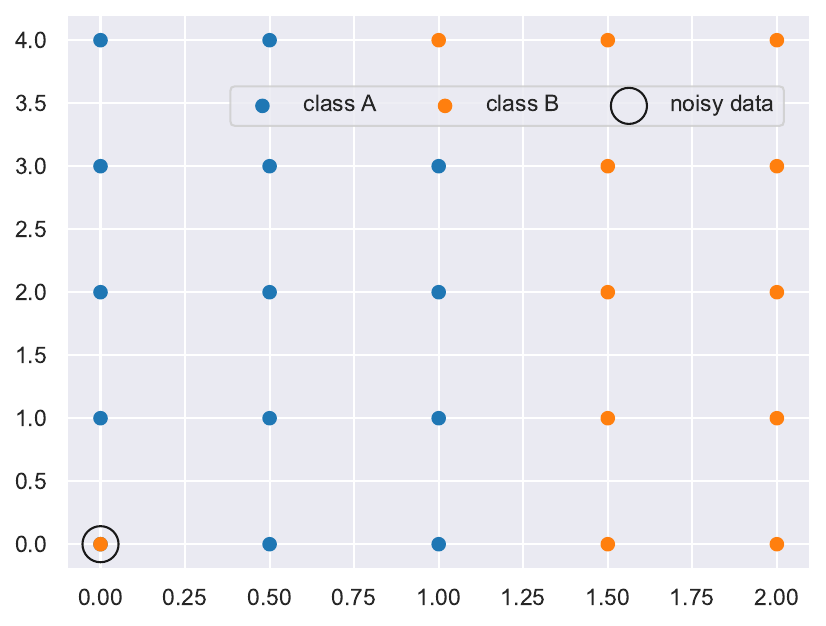}}
  \caption{An example of label noise.}
    \label{fig:label-noise-example2}
\end{figure}

Before we prove Theorem~\ref{thm:generalization-error-bound}, we need a lemma on cross-entropy.

\begin{lem}\label{lem:cross-entropy-expand}
    Consider a classification problem of $C$ classes. Given a data distribution $\pi$ such that $(x,y)\sim\pi,y\in[1:C]$, a neural network $f$ and a probability measure $\mathrm{Pr}$, then the expectation of cross-entropy loss is
    \begin{equation}
        -\sum_{i=1}^{C}\mathrm{Pr}_{\pi}(Y=i)\mathbb{E}_{X|Y=i}\left[f_i(X)-\log\left(\sum_{r=1}^{C}\exp(f_r(X))\right)\right]
    \end{equation}
\end{lem}

In most machine learning tasks, it is reasonable to assume that the input and output of the model are bounded, which we formalize in Assumptions~\ref{asp:bdd-x} and ~\ref{asp:bdd-output}.

\begin{asp}[Bounded input space]\label{asp:bdd-x}
    The input space $\mathcal{X}$ is bounded in $[0,1]^{d_x}\subset\mathbb{R}^{d_x}$.
\end{asp}

\begin{asp}[Bounded model output]\label{asp:bdd-output}
    Consider a neural network $f:\mathbb{R}^{d_x}\times\mathcal{W}\rightarrow\mathbb{R}^{d_y}$. We assume that its range $f(\mathbb{R}^{d_x};\mathcal{W})$ is bounded in $\mathbb{R}^{d_y}$. That is, $\exists C_f\ge 0$ such that $\forall x\in\mathbb{R}^{d_x},\forall\theta\in\mathcal{W},\forall i\in\{1,\dots,C\},C_f\ge|f_i(x;\theta)|$.
\end{asp}

Note that the upper bound of model output could change as we train the model for more epochs. To model the evolution of the output upper bound, we can relax Assumption~\ref{asp:bdd-output} and study a specific family of classifiers: ReLU networks. Later we can bound the generalization error evolution given the growth of path-norm proxy through iterations.

\begin{prop}[Polynomial growth of path-norm proxy]\label{prop:poly-pn}
    Consider an FL process with a $L$-layer neural network $f:\mathbb{R}^{d_x}\times\mathcal{W}\rightarrow\mathbb{R}^{d_y}$ as its global model, then its path-norm increases at most polynomially,
    \begin{equation}
        \lVert f(\cdot;\theta(t))\rVert_{\mathrm{pnp}}=\mathcal{O}(t^{L+1}E^{(L+1)/2})
    \end{equation}
    where $t\le R$ denotes the number of communication rounds and $E$ denotes the local training time.
    
    If we consider a generic decentralized algorithm, we have
    \begin{equation}
        \lVert f(\cdot;\theta(t))\rVert_{\mathrm{pnp}}=\mathcal{O}(e^{C't(L+1)}E^{(L+1)/2})
    \end{equation}
    where $C'$ is a constant independent of $t,L,E$.
\end{prop}

\noindent
\textbf{Interpretation of Proposition~\ref{prop:poly-pn}:} By Corollary 3.14 in \cite{b28}, small path-norm value guarantees an ``easier'' hypothesis space. Note that our upper bound on path-norm proxy is independent of dataset statistics and label noise.

\begin{figure}[H]
    \centering
    \includegraphics[width=7cm]{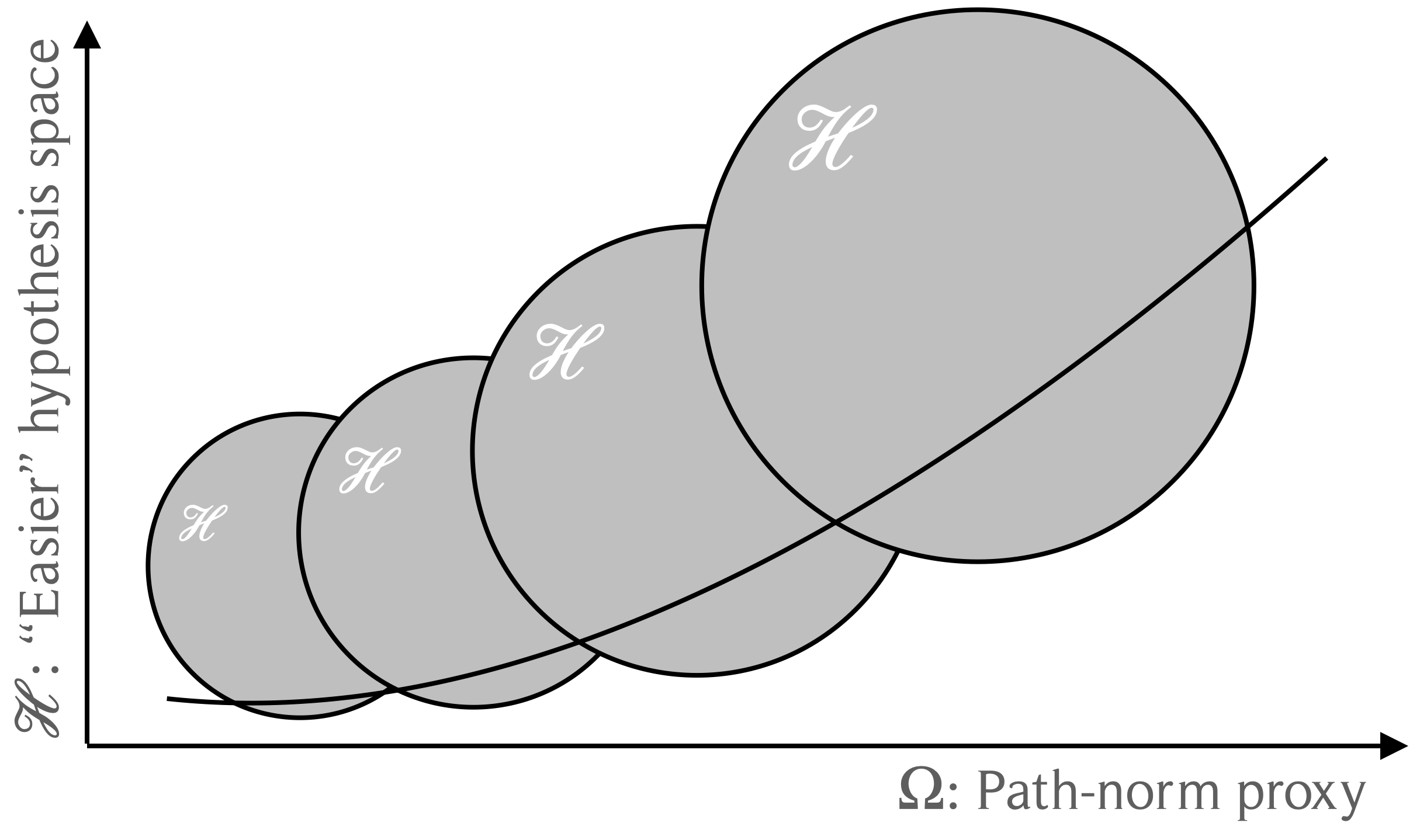}
    \caption{Growth of path-norm and resulted change of hypothesis space}
\end{figure}

\begin{cor}\label{cor:explicit-generror-bound}
    We can specify $\Omega$ in Theorem~\ref{thm:generalization-error-bound} with various assumptions:
    \begin{enumerate}
        \item By Assumption~\ref{asp:bdd-output}, $\Omega=C_f$.
        \item If we use a ReLU network as our model in the FL task, then $\Omega=\lVert f(\cdot;\theta(t))\rVert_{\mathrm{pnp}}$.
        \item By Assumption ~\ref{asp:bdd-x} and Proposition ~\ref{prop:poly-pn},
        \[
            \Omega=C_0 t^{L+1}E^{(L+1)/2}
        \]
        where $C_0$ is a constant independent of $t,E,L$.
    \end{enumerate}
\end{cor}
There are some important implications behind Corollary~\ref{cor:explicit-generror-bound}.
\begin{itemize}
    \item Since the first two statements of the corollary do not rely on the aggregation mechanism of the algorithm, they could also be extended from FL to a decentralized learning scenario, e.g. Swarm learning in decentralized clinical ML \cite{swarm_learning}, decentralized optimization algorithms \cite{decentralized_opt_ex,decentralized_extragradient}, ML on blockchain \cite{ML_blockchain}.
    \item Theorem~\ref{thm:generalization-error-bound} does not characterize the upper bound with communication rounds and local epochs in its general form. But it is a symbolic and concise term that helps us understand the impact of label noise. Nonetheless, case 3 in Corollary~\ref{cor:explicit-generror-bound} provides the interplay between the label noise, communication rounds, and local epochs.
\end{itemize}

\section{Numerical Results}\label{sec:numerical-result}

We present three numerical experiments to validate our theoretical results and draw new insights. We first verify our theoretical work on the path-norm proxy. Then we show experiments of 2-client, 4-client, and 15-client FL settings.

Our main findings are 1) the growth of path-norm proxy empirically increases in a \textbf{polynomial} order in FL; 2) there exists an approximate negative \textbf{linear} relation between the test accuracy of global model and the number of incorrectly labeled data; 3) label noise slows down the \textbf{convergence} of FL algorithms and induces \textbf{over-fitting} to the global model.

% \begin{table}[htbp]
% \caption{FL hyper-parameters}
% \begin{center}
% \begin{tabular}{|c|c|}
% \hline
% \textbf{Communication round} & 40\\
% \hline
% \textbf{Global learning rate} & 1.0\\
% \hline
% \textbf{Local learning rate} & 0.1\\
% \hline
% \textbf{Local epoch} & 5\\
% \hline
% \end{tabular}
% \label{tab1}
% \end{center}
% \end{table}

\subsection{Path-norm Proxy}

In this subsection, we study the path norm proxy and observe its relation with the number of layers and communication rounds.

\begin{figure}
    \centering
    \subfloat[\label{subfig:4client-PN-alg}]{%
       \includegraphics[width=0.5\linewidth]{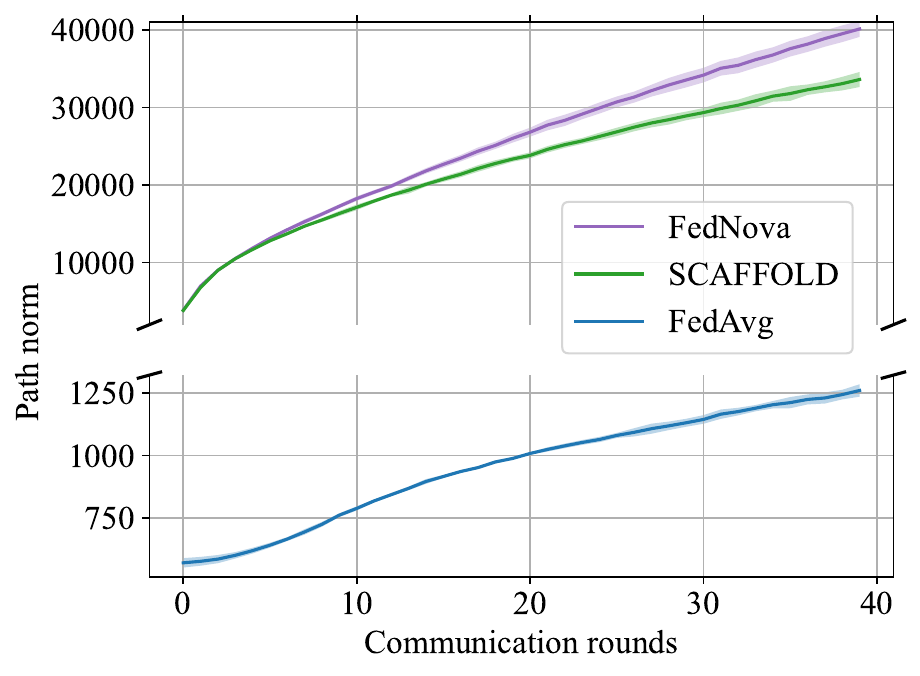}}
  \subfloat[\label{subfig:4client-PN-layer}]{%
        \includegraphics[width=0.5\linewidth]{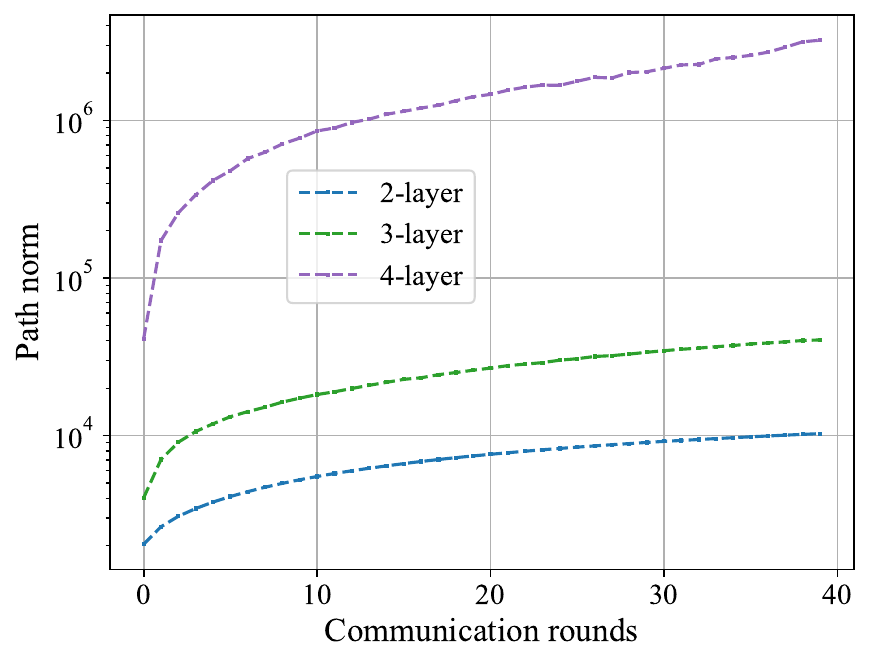}}
  \caption{A case study on MNIST dataset. (a) The path-norm generated by different FL algorithms with 3-layer ReLU network. (b) The path-norm generated by FedAvg algorithm and ReLU networks with different numbers of layers.}
    \label{fig:4clientpn}
\end{figure}

\noindent
\textbf{Compare different FL algorithms.} We use the same neural network structure and consider $N=4$ clients. We study three FL algorithms, FedAvg \cite{b1}, SCAFFOLD, and FedNova on MNIST dataset. We observe a concave growth of the global model's path-norm in Figure~\ref{subfig:4client-PN-alg}. We can sort the three algorithms according to the scale of their path-norm values
\[
    \text{Path norm value: FedNova} > \text{SCAFFOLD} >> \text{FedAvg}
\]
which corresponds to their different usage of local gradient $\mathrm{grad}$ and local parameter change $\Delta w$

\begin{center}
    \begin{tabular}{|c|c|}
        \hline
        Algorithm & Aggregation input\\
        \hline
        FedNova & $\mathrm{grad}$\\
        \hline
        SCAFFOLD & $\mathrm{grad},\Delta w$\\
        \hline
        FedAvg & $\Delta w$\\
        \hline
    \end{tabular}
\end{center}

These results empirically show that the path norm increases polynomially in FL.

\noindent
\textbf{Compare ReLU networks with different numbers of layers.} We train ReLU networks using FedAvg on MNIST. We use different numbers of layers in $\{2,3,4\}$. The result is illustrated in Figure~\ref{subfig:4client-PN-layer}. To verify the polynomial rule, we use a logarithmic scale on the y-axis. The three path-norm curves have similar shapes and almost differ up to a constant factor. This result is consistent with Proposition~\ref{prop:poly-pn}.

\subsection{Pilot experiments}

We run 2-client experiments with FedAvg algorithm on MNIST dataset. We study a 2-client setting for multiple considerations.
\begin{itemize}
    \item 2-client setting exists in practice. In cross-silo FL, clients could be enterprises, and each client could provide abundant data, so the total number of clients is relatively small. For example, since 2019, two insurance companies, Swiss Re and WeBank, have collaborated on federated learning \cite{b58}.
    \item This experiment serves as a starting point and gives us a thorough pedagogical understanding of the impact of label noise. We will study the 4-client and 15-client cases in the next subsection.
\end{itemize}

We generate the local datasets for two clients by dividing the whole dataset into two equally-sized parts. We add label noise to local datasets by uniformly flipping some instances' labels to other class labels. Each client has different noise levels. Denote the noise level of client $i$ as $\mathrm{wp}_i$, then
$(\mathrm{wp}_1,\mathrm{wp}_2)\in\{0\%,10\%,20\%,\dots,80\%,90\%\}^2$. Pathological noise levels (greater than $50\%$) have been studied in supervised learning settings \cite{AAAI2022_90percentNoise}.
We illustrate the test accuracy of global models under different degrees of label noise as bar charts in Fig~\ref{fig:2client-thing}. To verify the linear trend of test accuracy, we perform linear regression and visualize the result in Fig~\ref{fig:2client-regression}.

\begin{figure}[H]
    \centering
    \setkeys{Gin}{width=0.33\linewidth}
    \subfloat[\label{subfig:2client-thing}]{\includegraphics{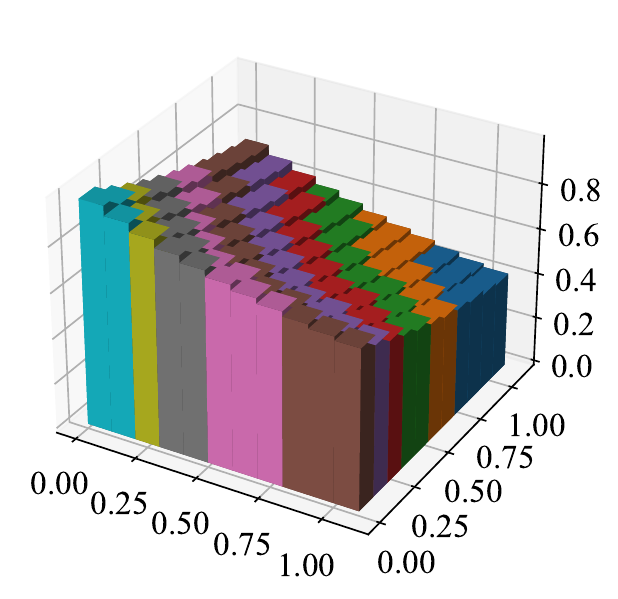}}
    \subfloat[\label{subfig:2client-fixclient1}]{\includegraphics{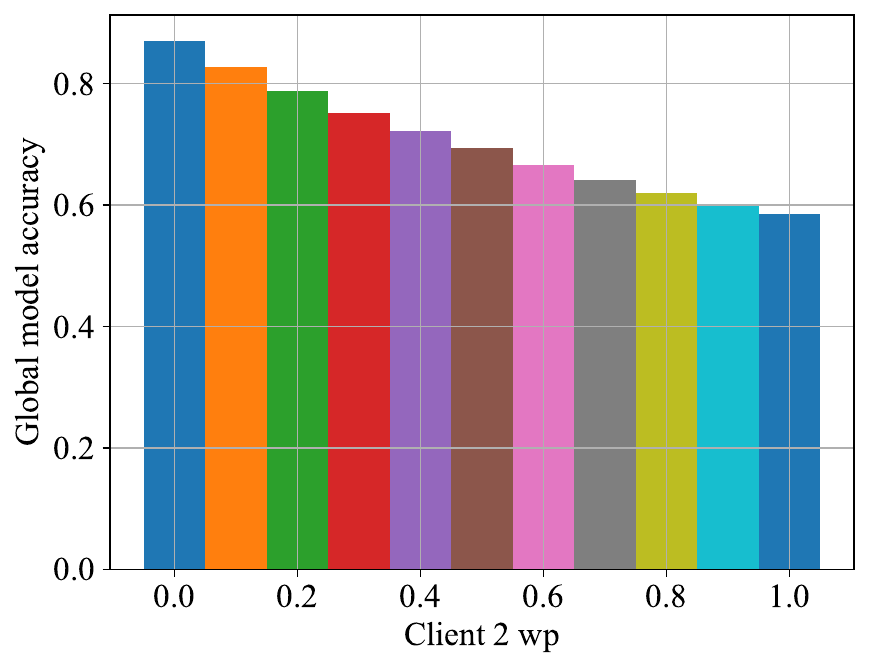}}
    \subfloat[\label{subfig:2client-fixclient2}]{\includegraphics{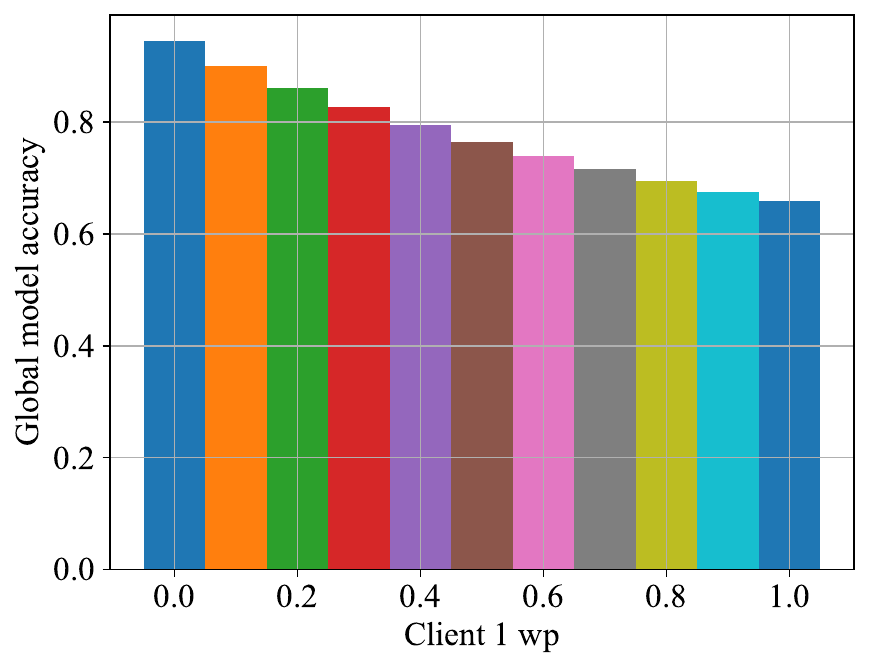}}
    \caption{(a) Bar plot of global model accuracy. x,y axes control the levels of label noise of each client. z axis represents the test accuracy of the global model; (b) Slice of bar plot when client $1$ has $30\%$ of wrongly labelled data; (c) Slice of bar plot when client $2$ has $10\%$ of wrongly labelled data.}
    \label{fig:2client-thing}
\end{figure}

\begin{figure}
\centering
\setkeys{Gin}{width=0.9\linewidth}
    \subfloat{\includegraphics{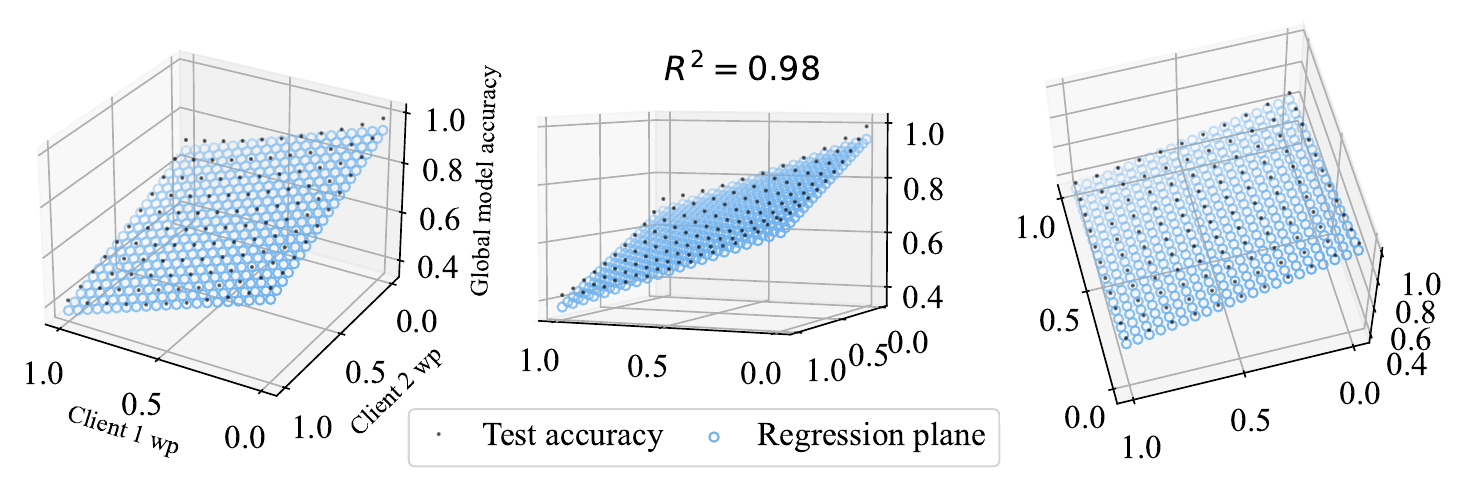}}
\caption{Linear regression on the global model accuracy.}
\label{fig:2client-regression}
\end{figure}

\begin{figure}
\centering
\setkeys{Gin}{width=0.33\linewidth}
    \subfloat[FedAvg\label{subfig:4client-client1-fedavg}]{\includegraphics{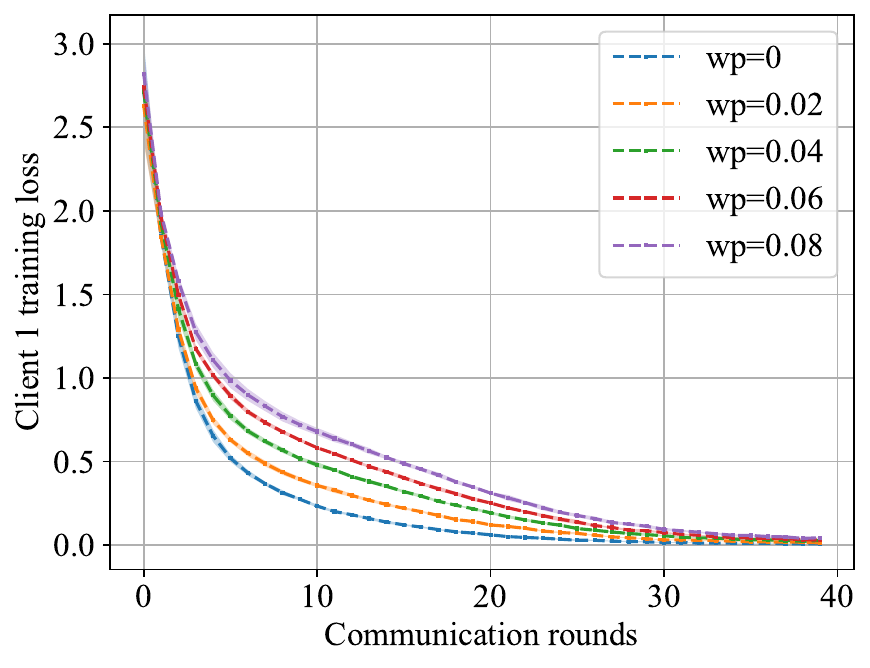}}
    \subfloat[SCAFFOLD\label{subfig:4client-client1-scaffold}]{\includegraphics{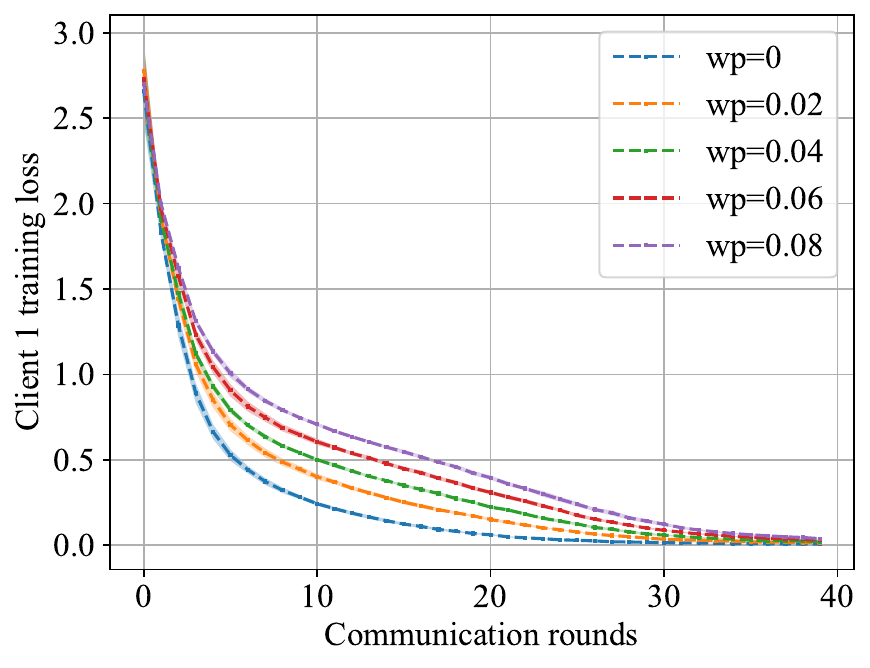}}
    \subfloat[FedNova\label{subfig:4client-client1-fednova}]{\includegraphics{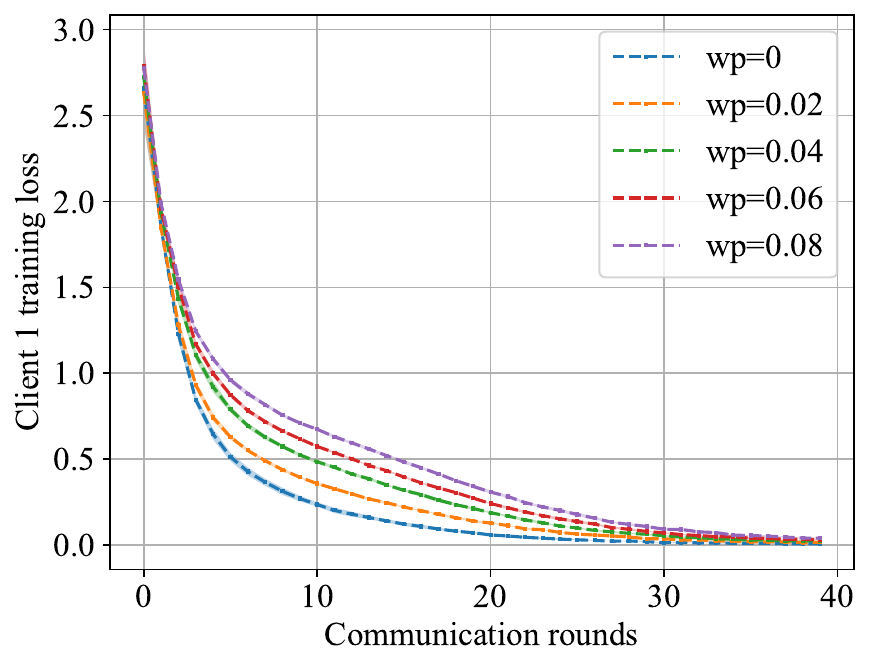}}
\caption{Training loss of Client 1 for $0\%,2\%,4\%,6\%,8\%$ percentages of wrongly labelled data (4-client setting).}
\label{fig:4client-client1-loss}
\end{figure}

\begin{figure}
\centering
\setkeys{Gin}{width=0.33\linewidth}
    \subfloat[FedAvg\label{subfig:15client-client1-fedavg}]{\includegraphics{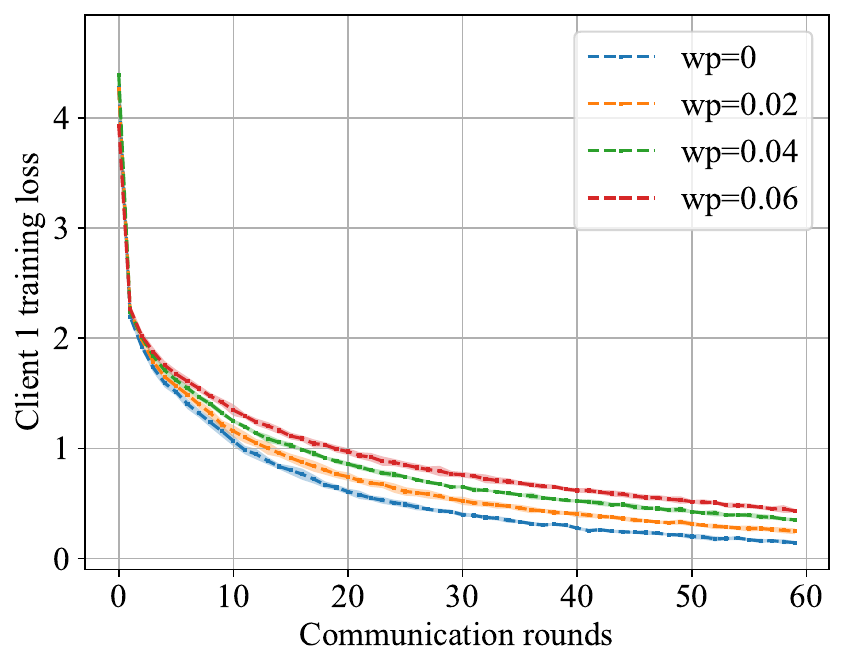}}
    \subfloat[SCAFFOLD\label{subfig:15client-client1-scaffold}]{\includegraphics{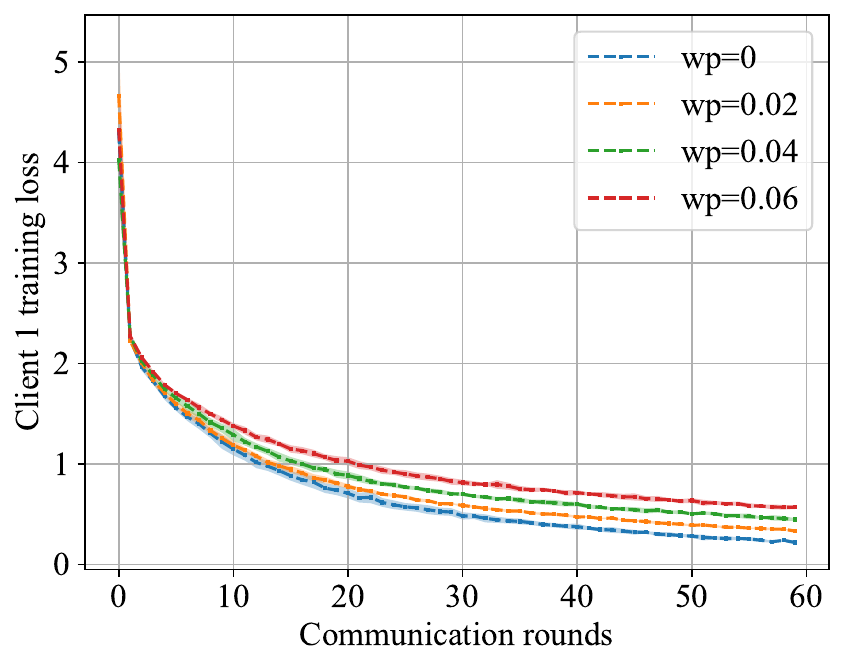}}
    \subfloat[FedNova\label{subfig:15client-client1-fednova}]{\includegraphics{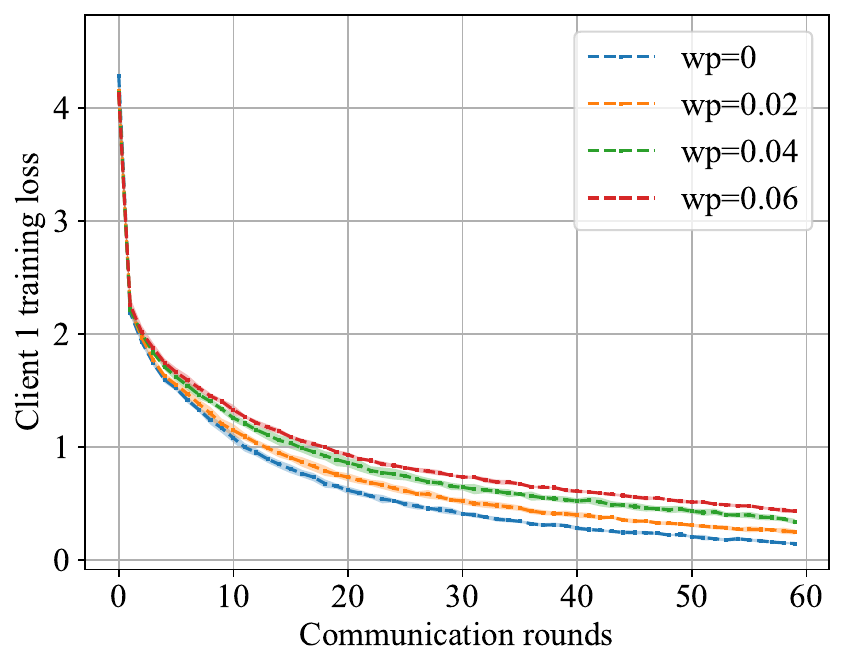}}
\caption{Training loss of Client 1 for $0\%,2\%,4\%,6\%,8\%$ percentages of wrongly labelled data (15-client setting).}
\label{fig:15client-client1-loss}
\end{figure}

\begin{figure}
\centering
\setkeys{Gin}{width=0.33\linewidth}
    \subfloat[FedAvg]{\includegraphics{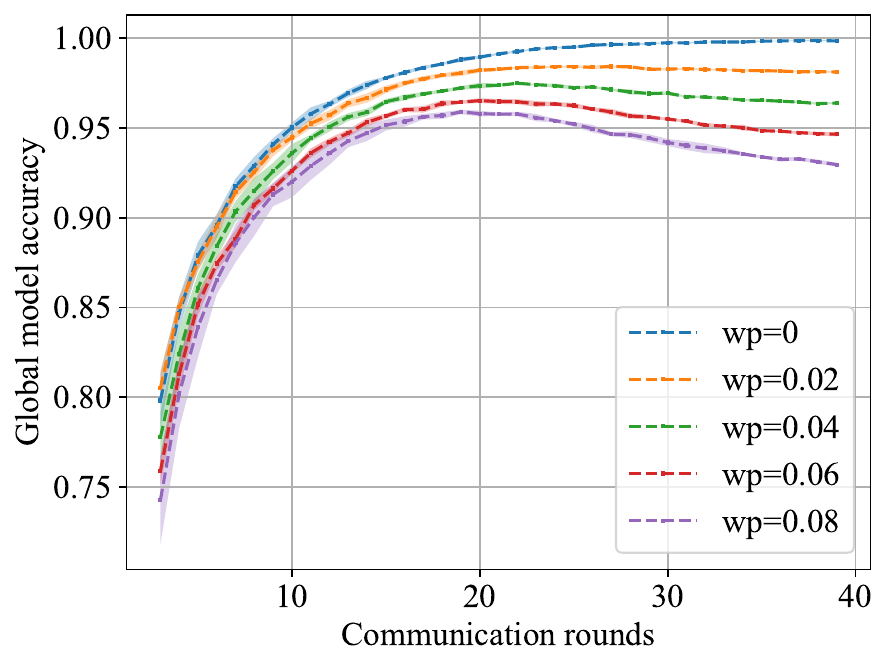}}
    \subfloat[SCAFFOLD]{\includegraphics{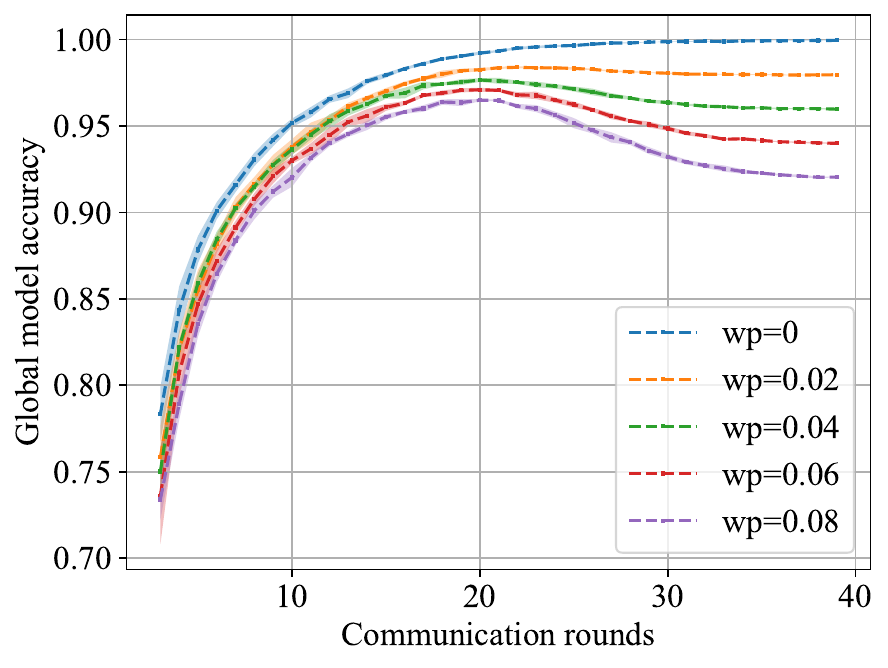}}
    \subfloat[FedNova]{\includegraphics{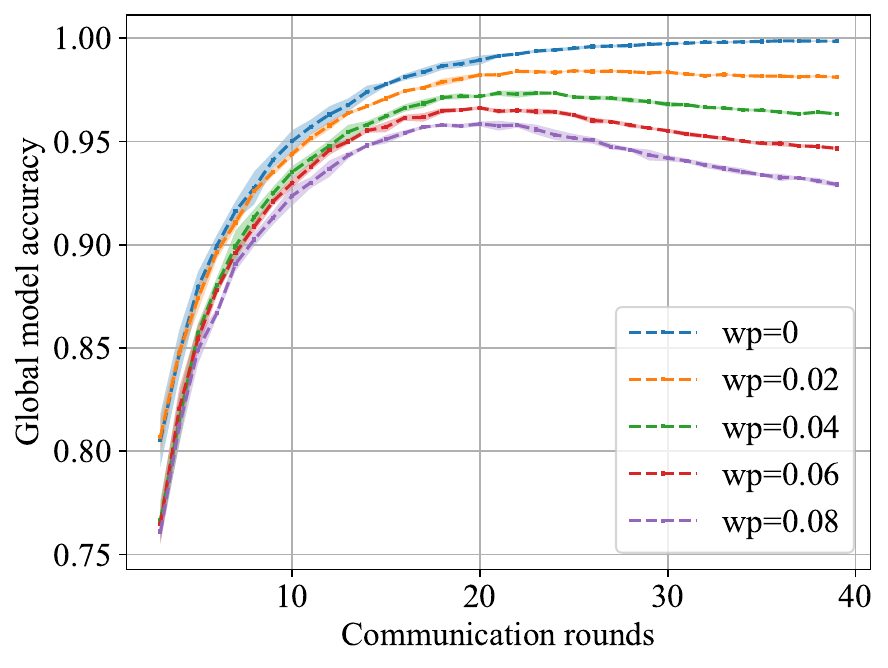}}
\caption{Test accuracy of global model for $0\%,2\%,4\%,6\%,8\%$ percentages of wrongly labelled data (4-client setting).}
\label{fig:4client-cvg}
\end{figure}

\begin{figure}
\centering
\setkeys{Gin}{width=0.33\linewidth}
    \subfloat[FedAvg]{\includegraphics{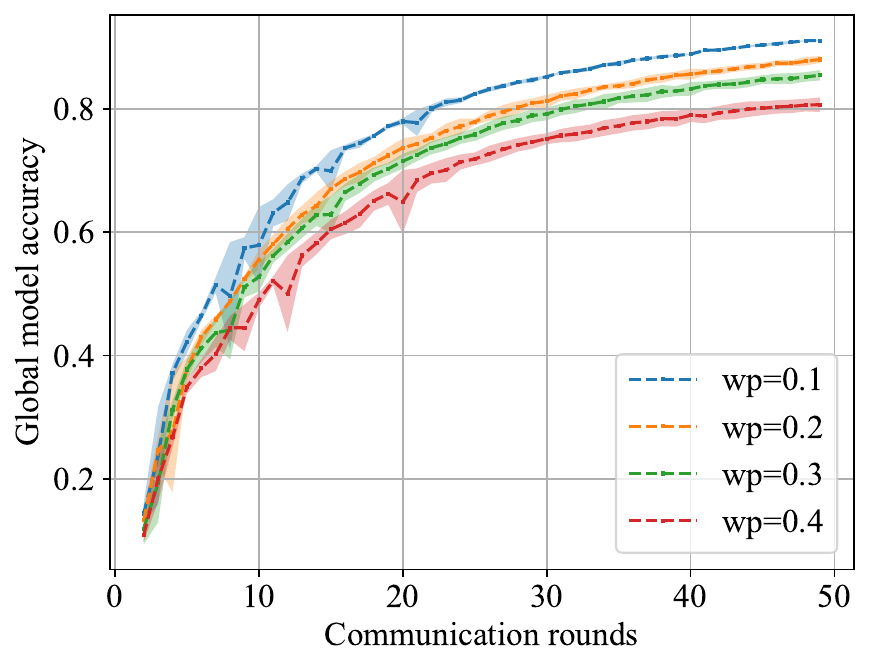}}
    \subfloat[SCAFFOLD]{\includegraphics{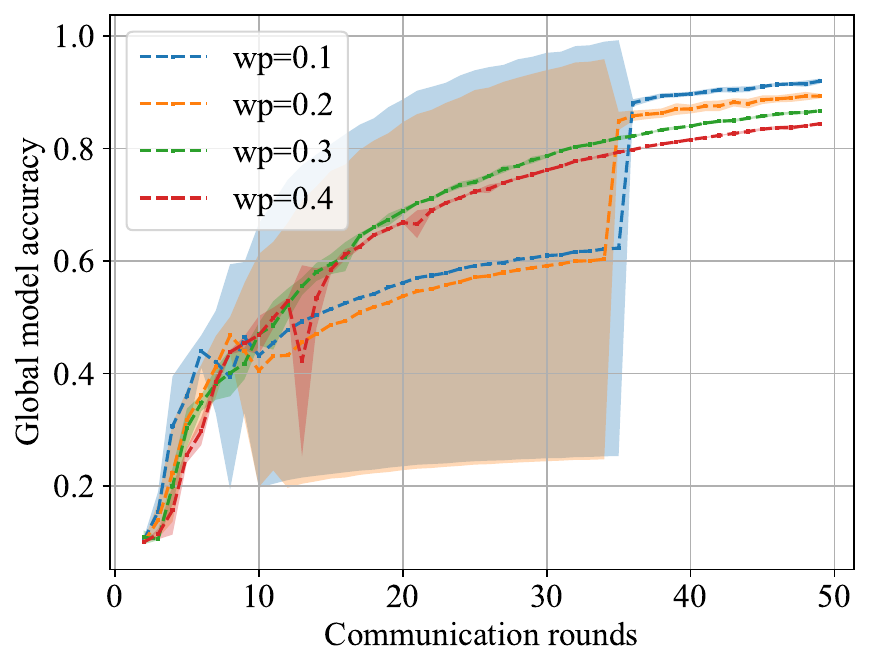}}
    \subfloat[FedNova]{\includegraphics{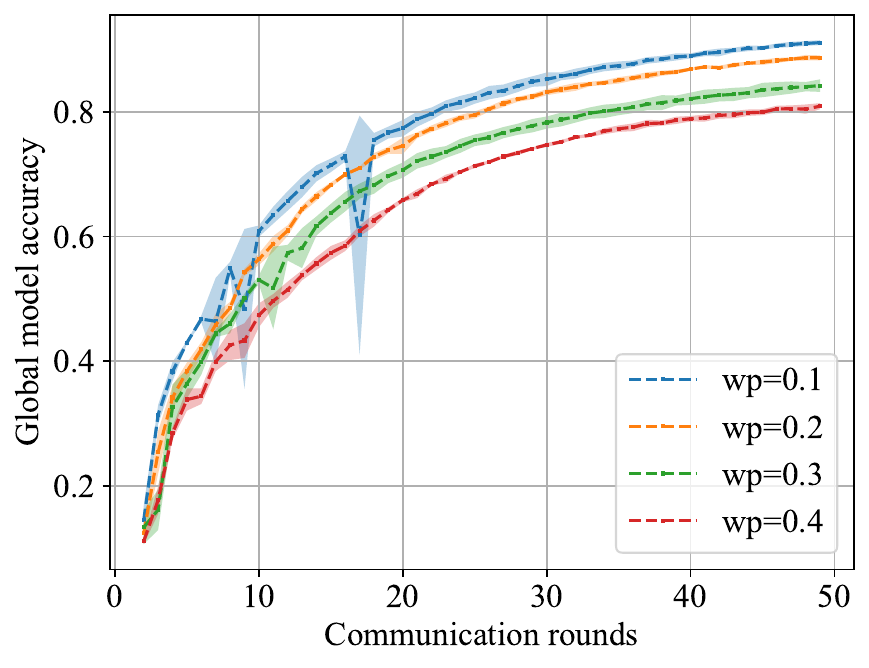}}
\caption{Test accuracy of global model for $1\%,2\%,3\%,4\%$ percentages of wrongly labelled data (30-client setting).}
\label{fig:30client-cvg}
\end{figure}

\begin{figure}
    \centering
    \setkeys{Gin}{width=0.33\linewidth}
    \subfloat[4-client]{\includegraphics{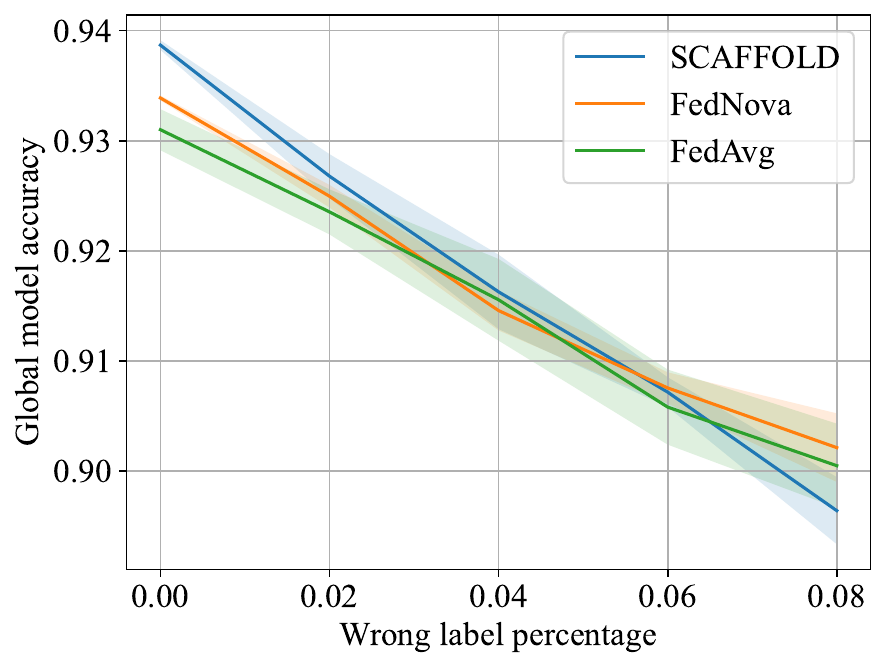}}
    \subfloat[15-client]{\includegraphics{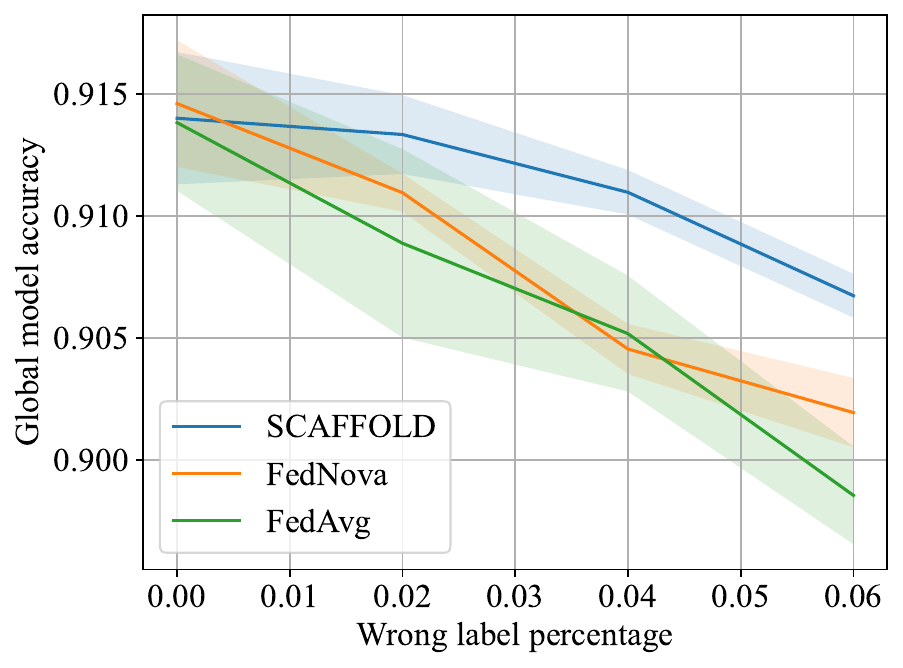}}
    \subfloat[30-client]{\includegraphics{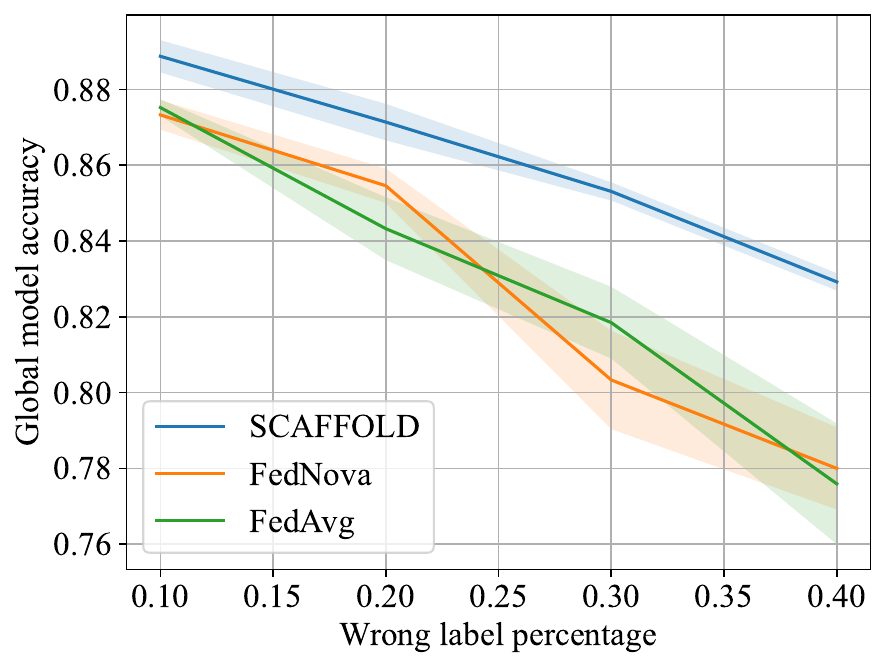}}
    \caption{Test accuracy of global model by different FL algorithms under different label error rates.}
    \label{fig:4client-wacc-alg}
\end{figure}

\noindent
\textbf{Negative bilinear trend by label noise.} Figure~\ref{fig:2client-thing} shows a negative bilinear relation between the test accuracy of the global model and noise label. When we apply linear regression on the test accuracy of the global model and the proportion of wrongly labeled data, we obtain a coefficient of determination of $0.98$ in Figure~\ref{fig:2client-regression}. That means the relation between the test accuracy and label noise has a strong linear relation.

\subsection{Experiments with larger cohort size}
\label{subsec:4client-experiment}

We run experiments on CIFAR-10 dataset respectively with four clients and fifteen clients. Local datasets are generated by dividing the whole dataset into equally-sized parts. We add label noise to local datasets by uniformly flipping some instances' labels to other class labels. In a case study by Gu et al., the real human annotation has a rater error rate of around $4.8\%$ \cite{b57}. Therefore it is reasonable to study the error rate within a relatively small range that contains $4.8\%$, i.e., from $0\%$ to $10\%$. We set the same proportion of wrongly labeled data for each client ($0\%,2\%,4\%,8\%$).

\noindent
\textbf{Slow Convergence by label noise.} In Figure~\ref{fig:4client-client1-loss} and Figure~\ref{fig:15client-client1-loss}, we plot how client 1's local model loss depends on the communication rounds at different percentages of wrongly labeled data. The training loss decreases slower with a larger proportion of wrongly labeled data, i.e., the algorithm converges slower with a larger proportion of wrongly labeled data.

\noindent
\textbf{Overfitting by label noise.} We observe in Figure~\ref{fig:4client-cvg} that for all three algorithms, the global model's test accuracy decreases after 20 communication rounds. The global model is more over-fitted with a larger percentage of wrongly labeled data. This result provides an engineering insight in FL that the over-fitting of the global model could result from some wrongly labeled data in the local datasets. It also motivates the study of mitigating label noise in FL \cite{b46}.

% \noindent
% \textbf{Compare different FL algorithms.} Figure~\ref{fig:4client-wacc-alg} shows a comparison for test accuracy of global models of FedAvg, SCAFFOLD, and FedNova. SCAFFOLD achieves the best test accuracy than other algorithms with smaller label noise, while FedNova achieves the best test accuracy with larger label noise.

\noindent
\textbf{Negative linear trend by label noise.} In Figure~\ref{fig:4client-wacc-alg}, all three algorithms show a negative linear relation between the test accuracy of the global model and the proportion of wrongly labeled data. This is consistent with our theoretical analysis. We also observe piece-wise linear trend in the experiments. The model accuracy decreases more when the total noise level exceeds certain threshold.

\subsection{Experiments with both label imbalance and instance-dependent label errors}

We run experiments on CIFAR-10 dataset with thirty clients (Figure~\ref{fig:30client-cvg}). Local datasets are generated by dividing the whole dataset with label imbalance. First, we generate the label imbalance with a symmetric Dirichlet distribution $\mathrm{Dir}(\alpha=10)$. Then we set the error ratio in the range of $0.1$ to $0.4$ and add label noise to local datasets with an instance-dependent error generator. Here we use a classifier (a ResNet-18 trained on CIFAR-10 dataset with a test accuracy 78.74\%) as our error generator. The classifier learned to correctly classify some easy instances while failing on the difficult ones. So the label errors the classifier generates are instance-dependent and more realistic.

\section{Discussions}

In this section, we discuss the limitation and potential application of our work.

\noindent
\textbf{Improving theoretical bounds:} We prove a linear upper bound for the generalization error, and the bound is consistent with numerical results. However, the upper bound can be loose. One can provide a lower bound or improve the upper bound by making more restrictive assumptions. For example, one can consider a regression task with MSE loss function that provides nicer theoretical properties \cite{jasonlee-label-noise}.

\noindent
\textbf{More comprehensive experiments:} Our experiments use a small number of clients, which applies to cross-silo FL. In future research, we plan to study the impact of label noise with a larger number of clients (e.g., as in cross-device FL).

\noindent
\textbf{Application:} Our results potentially serve as ``domain knowledge'' to improve FL algorithm design. Our work could also be used in designing incentive mechanisms in FL systems \cite{b34}. In particular, the qualitative relation in this paper helps model the performance of global model under label noise.

\noindent
\textbf{Methodology:} The emergence of large machine learning models has shifted the nature of AI research from an engineering science (iteratively improving models) to a natural science (probing capabilities of the models we designed) \cite{b26}. Researchers have been proposing hundreds of new models/algorithms for different AI problems. However, more must be done to understand how and why a proposed model/algorithm performs in a certain way. We must build theories based on observation and experiments to understand these artificial black boxes. In this way, we can transform AI research from engineering alchemy to white-box chemistry. Our work follows this scientific paradigm shift and conducts a case study for different FL models under label noise.

\section{Concluding Remarks}\label{sec:conclusion}

This paper takes the first step to quantify the impact of label noise on the global model in FL. The critical challenge is that we have little knowledge of the underlying information related to local data distributions and we do not have an explicit expression of the outcome of an FL algorithm. We show with both empirical evidence and theoretical proof that 1) label noise linearly degrades the global model's performance in FL; 2) label noise slows down the convergence of the global model; 3) label noise induces overfitting to the global model. Our results could provide insights into the design of a noise-robust algorithm and the design of an incentive mechanism.

\appendix

\section{Acknowledgments}
The work was partially supported through grant USDA/NIFA 2020-67021-32855, and by NSF through  IIS-1838207, CNS 1901218, OIA-2134901.

\bibliography{refs}

\section*{Appendix I}

% Data heterogeneity (or non-IID) corresponds to the case where clients' data distributions are different. 
In Federated Learning, the ``non-IID'' issue is defined as the statistical difference or statistical dependence of different local datasets from different clients \cite{b19,b21}. In this work, we consider that for different clients $i$ and $j$, the distributions $\pi_i,\pi_j$ of their local datasets are different
\[
    \pi_i\not=\pi_j
\]

There are further two major types of non-IID: feature distribution skew and label distribution skew.

\begin{itemize}
    \item For label distribution skew, all clients share the conditional probability $\pi_i(x|y)$, that is,
    \[
        \pi_i(x|y)=\pi_j(x|y),\forall i,j\in[1,2,\dots,N],\forall (x,y)\in\mathbb{R}^{d_x+d_y}
    \]
    while clients have different label distributions $\pi_i(y)$.
    \item For feature distribution skew, clients share the same conditional probability $\pi_i(y|x)$ while clients have different feature distributions $\pi_i(x)$.
\end{itemize}

\section*{Appendix II}

The ReLU network is a powerful prototype model among various types of neural network for its successful performance in different fields, including image classification and natural language processing \cite{b27}. In this work, we study ReLU network as a sub-case in generalization error analysis.

We represent a $L$-layer neural network as a map $f:\theta\mapsto f(\cdot;\theta)$ where $\theta$ denotes the weight of the network and $f(\cdot;\theta):x\mapsto f(x;\theta)$ is a function that maps an input $x\in\mathbb{R}^{d_{X}}$ to an output $y\in\mathbb{R}^{d_{Y}}$ of the network.

Denote the width of the $l$-th layer as $d_l$ where $1\le l\le L$ and let $d_{0}=d_x+1,d_{L+1}=d_Y$, i.e. there are $d_l$ nodes in the $l$-th layer. Here we simplify the notation by converting the affine map to a linear one: identify $x\in\mathbb{R}^{d_x}$ with $(x,1)\in\mathbb{R}^{d_x+1}$ \cite{b28}.

\begin{defi}[Rectified linear unit]
    The rectified linear unit function is defined to be
    \begin{equation}
        \sigma:\mathbb{R}\rightarrow\mathbb{R}_{\ge 0},\quad\sigma(x)=\max(0,x)
    \end{equation}
\end{defi}

\begin{defi}[Rectifier activation function]
    By an abuse of notation, we define the rectifier activation function by applying the rectified linear unit function element-wise
    \begin{equation}
        \sigma:\mathbb{R}^d\rightarrow\mathbb{R}_{\ge 0}^d,\quad\sigma\left(\begin{bmatrix}
            x_1\\\vdots\\x_d
        \end{bmatrix}\right)=\begin{bmatrix}
            \max(0,x_1)\\\vdots\\\max(0,x_d)
        \end{bmatrix}
    \end{equation}
\end{defi}

The $k$-th entry of output $f(x;\theta)$ of a $L$-layer ReLU network is \cite{b27,b28}

\medskip

\begin{equation}
    \resizebox{\hsize}{!}{
        $\begin{aligned}
        f_k(x;\theta)=&\sum_{i_L=1}^{d_L}\theta_L(i_L,k)\sigma\left(\sum_{i_{L-1}=1}^{d_{L-1}}\theta(i_{L-1},i_L)\sigma\left(\sum_{i_{L-2}}\cdots\sigma\left(\sum_{i_1=1}^{d_1}\theta_1(i_1,i_2)\sigma\left(\sum_{i_0=1}^{d_0}\theta_0(i_0,i_1)x_{i_0}\right)\right)\right)\right)\\
        =&\sum_{(i_0,\dots,i_L)}\theta_L(i_L,k)\cdot\prod_{l=1}^{L}\theta_{l}(i_{l-1},i_l)\cdot\prod_{l=1}^{L}\mathbf{1}\left\{g^l_{i_l}(\theta,x)>0\right\}\cdot x_{i_0}
    \end{aligned}$}
\end{equation}

\medskip

\noindent
where parameter $\theta_l(i_{l},i_{l+1})$ refers to the weight of the edge connecting the $i_l$-th node in layer $l$ and the $i_{l+1}$-th node in layer $l+1$; $g^l_{i_l}$ denotes the output of $i_l$-th node in layer $l$.

\section*{Appendix III}

\begin{defi}[Cross-entropy loss]\label{def:cross-entropy}
    Given a neural network $f$, an input vector $x$, and the output vector $y=\mathrm{Softmax}(f(x))$, we define the cross-entropy loss as
    \begin{equation}
        -\sum_{i}y_i\log\left(\frac{\exp(f_i(x))}{\sum_j\exp(f_j(x))}\right)
    \end{equation}
    where subscript $i$ denotes the $i$-th entry of a vector.
\end{defi}

\begin{proof}[Proof of Theorem~\ref{thm:generalization-error-bound}]
    Without loss of generality, we prove the theorem given that we use a neural network as our classifier. By Lemma~\ref{lem:cross-entropy-expand}, we expand the formula of generalization error as follows
    \begin{equation}
        \begin{aligned}
        &G(W)=\left|L(W)-L^{\dag}(W)\right|\\
        =&\left|\sum_{k=1}^{N}\frac{n_k}{n}\sum_{i=1}^{C}\int_{\mathcal{X}}f_i(x)\left(d\mathrm{Pr}_{\pi_k}(x,y)-d\mathrm{Pr}_{\mu}(x,y)\right)\right|\\
        &\text{by Assumption~\ref{asp:annotation-skew}}\\
        =&\left|\sum_{k=1}^{N}\frac{n_k}{n}\mathbb{E}_{X}\left[\sum_{i=1}^{C}f_i(X)\left(\mathrm{Pr}_{\pi_k}(Y=i|X)-\mathrm{Pr}_{\mu}(Y=i|X)\right)\right]\right|\\
        \le&\sum_{k=1}^{N}\frac{n_k}{n}\mathbb{E}_{X}\left[\sum_{i=1}^{C}f_i(X)\left|\mathrm{Pr}_{\pi_k}(Y=i|X)-\mathrm{Pr}_{\mu}(Y=i|X)\right|\right]\\
        \le&\sum_{k=1}^{N}\frac{n_k}{n}\Omega\cdot\mathbb{E}_{X}\left[\sum_{i=1}^{C}\left|\mathrm{Pr}_{\pi_k}(y=i|X)-\mathrm{Pr}_{\mu}(y=i|X)\right|\right]
        \end{aligned}
    \end{equation}
\end{proof}

\begin{thm}[Growth of path-norm proxy {\cite[Corollary 5]{b28}}]\label{thm:relu-pn-grow-weinan}
    Consider an arbitrarily wide $L$-layer ReLU neural network. If the network's weight evolves under continuous gradient flow dynamics, then the network's path-norm increases at most polynomially
    \begin{equation}
        \lVert f(\cdot;\theta(t))\rVert_{\mathrm{pnp}}\le\left(C_0+\sqrt{\mathcal{R}(f(\cdot;\theta(t=0)))}t^{1/2}\right)^{L+1}
    \end{equation}
    where $\theta(t)$ denotes the weight of the network at time $t$, $C_0$ is a constant and $\mathcal{R}$ is the expectation of a sufficiently smooth loss function $\ell$
    \begin{equation}
        \mathcal{R}(\theta;f)=\int_{\mathbb{R}^{d_x+d_y}}\ell(f(x;\theta),y)\mathbb{P}(dx\otimes dy)
    \end{equation}
\end{thm}

\begin{proof}[Proof sketch of Proposition~\ref{prop:poly-pn}]
    This proof assumes the gradient-flow evolution and an arbitrarily wide neural network based on the proofs in \cite{b28}.
    
    We first consider an FL setting. Let $\theta^{(i)}(t_k)$ denote the collection of all the parameters of client $i$'s neural network uploaded to the central server at the $k$-th communication round before FL aggregation. Let $\theta_j^{(i)}(t_k)$ denote the collection of $j$-th layer parameters of client $i$'s neural network uploaded to the central server at the $k$-th communication round before FL aggregation. Let $\tilde{\theta}_j(t_k)$ denote the collection $j$-th layer parameters of the global neural network at the $k$-th communication round after FL aggregation, i.e.
    \begin{align*}
        \tilde{\theta}_j(t_k)=&\mathrm{aggregation}(\theta_j^{(1)}(t_k),\dots,\theta_j^{(N)}(t_k))\\
        =&\phi(\theta_j^{(1)}(t_k),\dots,\theta_j^{(N)}(t_k))
    \end{align*}

    Define the risk functional of client $i$:
    \[
        \mathcal{R}^{(i)}(\cdot):=\int_{\mathcal{X}}\int_{\mathcal{Y}}\ell(f(x;\cdot),y)d_{\pi_i}\mathrm{Pr}(x,y)
    \]
    where $\mathcal{X}$ denotes the feature/input space and $\mathcal{Y}$ denotes the label/output space.

    We first consider the path-norm evolution during local update of client $i$. Let $E$ denote the time range of local update under gradient flow and $t_k=t_{k+1}+E,\forall k\le R$. By Theorem~\ref{thm:relu-pn-grow-weinan}, we have
    \begin{align*}
        \lVert\theta_j^{(i)}(t_{k+1})\rVert_{L^2(\pi^{j+1}\otimes\pi^{j})}\le&\lVert\tilde{\theta}_j(t_k)\rVert_{L^2(\pi^{j+1}\otimes\pi^{j})}\\
        &+\sqrt{\mathcal{R}^{(i)}\left(\tilde{\theta}(t_k)\right)}E^{1/2}
    \end{align*}
    Without loss of generality, consider the FedAvg aggregation scheme, i.e. for $1\le k\le R$,
    \begin{align*}
        \tilde{\theta}(t_k)=&\phi(\theta^{(1)}(t_k),\dots,\theta^{(N)}(t_k))\\
        =&\frac{1}{N}\sum_{i=1}^{N}\theta^{(i)}(t_k)
    \end{align*}
    Let $\Omega_{j}$ denote the index space of the $j$-th layer of a neural network, and $\theta_j^{(i)}(w_{j+1},w_{j},t_{k+1})$ denote the weight of the neural network given indices $w_{j+1},w_{j}$, then
    \begin{align*}
        &\lVert\tilde{\theta}_j(t_{k+1})\rVert_{L^2(\pi^{j+1}\otimes\pi^{j})}\\
        =&\left\lVert\frac{1}{N}\sum_{i=1}^{N}\theta_j^{(i)}(t_{k+1})\right\rVert_{L^2(\pi^{j+1}\otimes\pi^{j})}\\
        =&\int_{\Omega_{j+1}\times\Omega_{j}}\left(\frac{1}{N}\sum_{i=1}^{N}\theta_j^{(i)}(w_{j+1},w_{j},t_{k+1})\right)^2dw_{j+1}w_{j}\\
        \le&\frac{1}{N}\sum_{i=1}^{N}\lVert\theta_j^{(i)}(t_{k+1})\rVert_{L^2(\pi^{j+1}\otimes\pi^{j})}\\
        \le&\frac{1}{N}\sum_{i=1}^{N}\left(\lVert\tilde{\theta}_j(t_k)\rVert_{L^2(\pi^{j+1}\otimes\pi^{j})}+\sqrt{\mathcal{R}^{(i)}\left(\tilde{\theta}(t_k)\right)}E^{1/2}\right)\\
        \le&\lVert\tilde{\theta}_j(t_k)\rVert_{L^2(\pi^{j+1}\otimes\pi^{j})}+\max_{1\le i\le N}\sqrt{\mathcal{R}^{(i)}\left(\tilde{\theta}(t_k)\right)}E^{1/2}\\
        \le&\lVert\tilde{\theta}_j(t_k)\rVert_{L^2(\pi^{j+1}\otimes\pi^{j})}+\max_{1\le i\le N}\sqrt{\mathcal{R}^{(i)}\left(\tilde{\theta}(t_0)\right)}E^{1/2}\\
        \le&\lVert\tilde{\theta}_j(t_0)\rVert_{L^2(\pi^{j+1}\otimes\pi^{j})}+(k+1)\max_{1\le i\le N}\sqrt{\mathcal{R}^{(i)}\left(\tilde{\theta}(t_0)\right)}E^{1/2}
    \end{align*}

    Now we have an upper bound of the $j$-th layer parameters, we can then derive the upper bound of the path-norm. By Lemma 4.6 in \cite{b28},
    \begin{align*}
        \lVert f\rVert_{\mathrm{pnp}}\le\left(C'+k\max_{1\le i\le N}\sqrt{\mathcal{R}^{(i)}\left(\tilde{\theta}(t_0)\right)}E^{1/2}\right)^{L+1}
    \end{align*}
    where $C'$ is a constant and $C'\ge\lVert\tilde{\theta}_j(t_{0})\rVert_{L^2(\pi^{j+1}\otimes\pi^{j})}$ for all $1\le j\le L$.

    Note that the above result also applies for SCAFFOLD with a new risk functional of client $i$:
    \[
        \mathcal{R}^{(i)}(\cdot;f)=\int_{\mathbb{R}^{d_x+d_y}}(\ell(f(x;\cdot),y)+\langle\cdot,c_i\rangle)\mathbb{P}(dx\otimes dy)+M
    \]
    where $c_i$ denotes the client control variate and $M$ is a constant. We introduce $M$ to ensure that the risk functional is always non-negative. The gradient-flow dynamics does not depend on the choice of $M$.

    Next, we consider a decentralized learning setting. Here we take the multi-consensus stochastic variance reduced extragradient algorithm as an example \cite{decentralized_extragradient}. The local update analysis is similar to FL. As for the aggregation step in decentralized optimization, the communication step is typically written as the matrix multiplication
    \[
        \tilde{\theta}(t_k)=W\theta(t_k)
    \]
    where $\theta$ denotes the matrix which collects all clients' parameter vector $\theta^{(i)}$. We assume that
    \begin{itemize}
        \item $W_{ij}\not=0$ if client $i$ and $j$ can exchange information;
        \item $W$ is a symmetric matrix;
        \item $\mathbf{0}\preceq W\preceq I,W\mathbf{1}=\mathbf{1},\mathrm{null}(I-W)=\mathrm{span}(\mathbf{1})$.
    \end{itemize}
    
    Let $\theta^{(i)}(t_k)$ denote the collection of all the parameters of client $i$'s neural network uploaded to the central server at the $k$-th communication round before decentralized communication. Let $\theta_j^{(i)}(t_k)$ denote the collection of $j$-th layer parameters of client $i$'s neural network uploaded to the central server at the $k$-th communication round before decentralized communication. Let $\tilde{\theta}^{(i)}_j(t_k)$ denote the collection $j$-th layer parameters of client $i$'s neural network at the $k$-th communication round after the decentralized communication.
    
    By Lemma 2 in \cite{decentralizedAggregation} and Lemma 2.1 in \cite{decentralized_extragradient}, we obtain a bound on the mixing rate of parameters,
    \begin{align*}
        &\left\lVert\tilde{\theta}_j^{(i)}(t_k)-\frac{1}{N}\sum_{I=1}^{N}\theta_j^{(I)}(t_k)\right\rVert_{L^2(\pi^{j+1}\otimes\pi^{j})}\\
        \le&\lambda_2(W)\left\lVert\theta_j^{(i)}(t_k)-\frac{1}{N}\sum_{I=1}^{N}\theta_j^{(I)}(t_k)\right\rVert_{L^2(\pi^{j+1}\otimes\pi^{j})}
    \end{align*}
    where $\lambda_2(W)$ denotes the second largest eigenvalue of $W$. Then
    \begin{align*}
        &\lVert\tilde{\theta}_j^{(i)}(t_{k+1})\rVert_{L^2(\pi^{j+1}\otimes\pi^{j})}\\
        \le&(1+2\lambda_2(W))\lVert\tilde{\theta}_j(t_k)\rVert_{L^2(\pi^{j+1}\otimes\pi^{j})}\\
        &(1+\lambda_2(W))\max_{1\le i\le N}\sqrt{\mathcal{R}^{(i)}\left(\tilde{\theta}(t_k)\right)}E^{1/2}\\
        \le&(1+2\lambda_2(W))\lVert\tilde{\theta}_j(t_k)\rVert_{L^2(\pi^{j+1}\otimes\pi^{j})}\\
        &(1+\lambda_2(W))\max_{1\le i\le N}\sqrt{\mathcal{R}^{(i)}\left(\tilde{\theta}(t_0)\right)}E^{1/2}\\
        \le&(1+2\lambda_2(W))^{k+1}\lVert\tilde{\theta}_j(t_0)\rVert_{L^2(\pi^{j+1}\otimes\pi^{j})}\\
        &+\frac{1+\lambda_2(W)}{2\lambda_2(W)}((1+2\lambda_2(W))^{k+1}-1)\max_{1\le i\le N}\sqrt{\mathcal{R}^{(i)}\left(\tilde{\theta}(t_0)\right)}E^{1/2}
    \end{align*}
    again, by Lemma 4.6 in \cite{b28}, we have
    \begin{align*}
        \lVert f\rVert_{\mathrm{pnp}}=\mathcal{O}\left(e^{C'' k}\max_{1\le i\le N}\sqrt{\mathcal{R}^{(i)}\left(\tilde{\theta}(t_0)\right)}E^{1/2}\right)^{L+1}
    \end{align*}
\end{proof}

\end{document}